\documentclass{article}

     \PassOptionsToPackage{numbers, compress}{natbib}
\usepackage[final]{neurips_2022}

\usepackage[utf8]{inputenc} % allow utf-8 input
\usepackage[T1]{fontenc}    % use 8-bit T1 fonts
\RequirePackage{titletoc}
\usepackage[colorlinks,citecolor=blue]{hyperref}       % hyperlinks
\usepackage{url}            % simple URL typesetting
\usepackage{booktabs}       % professional-quality tables
\usepackage{amsfonts}       % blackboard math symbols
\usepackage{nicefrac}       % compact symbols for 1/2, etc.
\usepackage{microtype}      % microtypography
\usepackage{xcolor}         % colors
\usepackage{amsmath}
\usepackage{amssymb}
\usepackage{amsthm} % proof environment
\usepackage{physics}
\usepackage{bbm}
\usepackage{dsfont}
\usepackage{calrsfs}
\usepackage{graphicx}
\usepackage{microtype} % improved typesetting
\usepackage{algorithm}
\usepackage{algpseudocode}
\usepackage{subcaption}
\usepackage{bm}
\usepackage{comment}
\usepackage{csquotes}
\usepackage{multirow}
\usepackage{setspace}
\usepackage{xr}
\MakeOuterQuote{"}
\usepackage{siunitx}

\DeclareMathOperator{\E}{\mathbb{E}}
\DeclareMathOperator*{\argmin}{arg\,min}
\newcommand{\one}{\mathbf{1}}
\newcommand{\prob}{\mathds{P}}

\newcommand{\reals}{\mathbb{R}}
\newcommand{\rset}{\mathbb{R}}
\newcommand{\naturals}{\mathbb{N}}
\newcommand{\expec}{\mathbb{E}}
\newcommand{\diag}{\operatorname{diag}}

\newtheorem{assumption}{Assumption}
\newtheorem{definition}{Definition}
\newtheorem{remark}{Remark}
\newtheorem{theorem}{Theorem}
\newtheorem{proposition}{Proposition}
\newtheorem{lemma}{Lemma}
\newcommand{\diff}{\mathrm{d}}

\newcommand{\js}[1]{\textcolor{magenta}{\small\sffamily\upshape [#1]}}

\title{A Quadrature Rule combining \\ Control Variates and Adaptive Importance Sampling}

\author{%
	R\'{e}mi Leluc \\
	LTCI, T\'{e}l\'{e}com Paris \\
	Institut Polytechnique de Paris, France \\
	\texttt{remi.leluc@telecom-paris.fr} \\
	\And
	Fran\c{c}ois Portier \\
	CREST \\
	ENSAI, France \\
	\texttt{francois.portier@gmail.com} \\
	\And
	Aigerim Zhuman \\
	LIDAM, ISBA\\
	UCLouvain, Belgium \\
	\texttt{aigerim.zhuman@uclouvain.be} \\
	\And
	Johan Segers \\
	LIDAM, ISBA\\
	UCLouvain, Belgium \\
	\texttt{johan.segers@uclouvain.be} \\
}

\begin{document}
	
	\maketitle
	
	\begin{abstract} Driven by several successful applications such as in stochastic gradient descent or in Bayesian computation, control variates have become a major tool for Monte Carlo integration. However, standard methods do not allow the distribution of the particles to evolve during the algorithm, as is the case in  sequential simulation methods. Within the standard adaptive importance sampling framework, a simple weighted least squares approach is proposed to improve the procedure with control variates. The procedure takes the form of a quadrature rule with adapted quadrature weights to reflect the information brought in by the control variates. The quadrature points and weights do not depend on the integrand, a computational advantage in case of multiple integrands.	Moreover, the target density needs to be known only up to a multiplicative constant.	Our main result is a non-asymptotic bound on the probabilistic error of the procedure. The bound proves that for improving the estimate's accuracy, the benefits from adaptive importance sampling and control variates can be combined. The good behavior of the method is illustrated empirically on synthetic examples and real-world data for Bayesian linear regression.
		%computational power of the approach is better when 
		%The approach has b computational properties when several integrals need to be computed and when the 
		%
		%
		%Most of the theory developed for the use of control variate is established for 
		%
		%
		%On the one hand, adaptive importance sampling is a well-known technique based on a re-weighting strategy to iteratively estimate the so-called target distribution. 
		%
		%On the other hand, control variates are a variance reduction method which may be implemented by the ordinary least squares estimator for the intercept in a multiple linear regression model with the integrand as response and the control variates as covariates. 
	\end{abstract}

\externaldocument{main_article}

\section{Introduction}
\label{section:introduction}
\vspace{-0.1cm}	
In recent years, sequential simulation has emerged as a leading approach to compute multidimensional integrals. A key object in sequential simulation is the sequence of distributions, called the policy, from which to generate the random variables, called particles, used to approximate the integrals of interest.
The policy is designed to evolve in the course of the algorithm to mimic the target density, which may itself be known only up to a proportionality constant.
While the design of algorithms with adaptive policies has been of major interest recently, only a few studies have focused on using control variates to reduce the variance. 
This paper provides a new method to incorporate control variates within standard sequential algorithms. The proposed approach significantly improves the accuracy of the initial algorithm, both theoretically and in practice.

	%supports for the claim that within most  standard sequential algorithm (associated to a given policy), control variates and show that the method diminishes the variance. 
		
\textbf{The sequential framework.}
Consider the problem of approximating the integral $\int g f \, \diff \lambda = \int_{\reals^d} g(x) f(x) \, \diff x$, where $\lambda$ is the $d$-dimensional Lebesgue measure, $f$ is a probability density on $\reals^d$ and the integrand $g$ is a real-valued function on $\reals^d$. For instance, one may think of $f$ as the posterior density in Bayesian inference. Let $(q_i)_{i\geq 0} $ be the policy of the algorithm, i.e., a sequence of probability densities which evolves adaptively depending on previous outcomes. The particles $(X_i)_{i\geq 1}$ are generated sequentially---at iteration $i$, particle $X_i$ is drawn from $q_{i-1}$. The integral $\int g f \, \diff \lambda$ is estimated by the normalized sum $\bigl( \sum_{i=1}^n w_i g(X_i) \bigr) / \bigl( \sum_{i=1}^n w_i \bigr)$, where $w_i = f(X_i) / q_{i-1} (X_i)$ are the importance weights. The normalization $\sum_{i=1}^n w_i$ allows to deal with situations where the target density $f$ is known only up to a proportionality constant.

Such an algorithm is part of the \emph{adaptive importance sampling} (AIS) framework. Many different ways have been investigated to update the densities $q_i$ adaptively. Early works that inspired such sequential schemes include \cite{geweke:1989,kloek+v:1978,oh1992adaptive} where the sampling policy is chosen out of a parametric family. The parametric approach has been further extended by the Population Monte Carlo framework \citep{cappe2008adaptive,cappe+g+m+r:2004,martino2017layered}. Various asymptotic results have been obtained in \cite{chopin:2004,douc+m:2008,portier2018asymptotic}. 
In \cite{dai2016provable,delyon2021safe,korba2022adaptive,zhang:1996}, \emph{nonparametric importance sampling} based on kernel smoothing is studied. The latter bears resemblance to \textit{sequential Monte Carlo} methods \citep{DelMoralDoucetJasra,chopin:2004}, in which the target distribution $f$ changes in the course of the algorithm. 
	
\textbf{Control variates.} 
Let $h = (h_1,\ldots, h_m)^\top$ be a vector of real-valued functions on $\reals^d$ such that for each $k$, the integral $\int h_k f \, \diff \lambda$ is known. Without loss of generality, suppose that $\int h f \, \diff \lambda = 0$. The functions $h_k$ are called control variates and can be obtained in different ways. In Bayesian statistics, Stein control variates \citep{oates2017control} are constructed by applying the second-order Stein operator to functions satisfying certain regularity conditions \citep{mira2013zero}. Other control variates might be created by re-weighting a function $h^*$ that satisfies $\int h^* \, \diff \lambda = 0$ via $h = h^* / f$. The use of control variates is a well studied variance-reduction technique \citep{glynn+s:2002,owen+z:2000}. The benefits can be established theoretically in terms of error bounds \citep{oates2017control,LelucPortierSegers2021}, weak convergence \citep{PortierSegers2019}, the excess risk \citep{belomestny2022empirical} and even uniform error bounds over large classes of integrands \citep{plassier2020risk}. In practice, the control variates framework has led to efficient procedures in reinforcement learning \cite{JieAbbeel2010,liu+f+m+z+p+l:2017} and optimization \cite{WangChenSmolaXing}, to name a few. Importance sampling and control variates in case of a Gaussian target density is explored in \cite{jourdain2009}.
The procedure in \cite{kawai2020} incorporates control variates and is said to involve adaptive importance sampling, but in fact the particles are always sampled from the uniform distribution on the unit cube. To the best of our knowledge, the existing control variate methods do not account for sequential changes in the particle distribution as is the case in AIS.
%\js{References to be checked: \url{https://doi.org/10.1016/j.jmaa.2019.123608}, \url{https://cermics.enpc.fr/~jourdain/radon.pdf}, \url{https://artowen.su.domains/pubtalks/AdaptiveISweb.pdf}}

\textbf{AISCV estimate.} 
The proposed approach to use control variates within the sequential AIS framework relies on the ordinary least squares expression of control variates (see for instance \cite{PortierSegers2019}). To take care of the policy changes, some re-weighting must be applied. The AISCV estimate of the integral $\int g f \, \diff \lambda$ is defined as the first coordinate of the solution to the weighted least squares problem
\vspace{-0.25cm}	
\[
	(\hat \alpha_n,  \hat \beta_n ) = 
	\argmin_{a \in \reals, b \in \reals^m} \sum_{i=1}^n w_i  \left( g(X_i) - a - b^\top h(X_i) \right)^2,
\]
with $w_i$ the importance weights from before.
The AISCV estimate $\hat{\alpha}_n$ has several interesting properties: (a) whenever $g$ is of the form $\alpha +  \beta^\top h$ for some $\alpha \in \reals$ and $\beta \in \reals^m$, the error is zero, i.e., $\hat{\alpha}_n = \alpha = \int g f \, \diff \lambda$; (b) the estimate takes the form of a  quadrature rule $\hat \alpha_n = \sum_{i=1}^n v_{n,i} g(X_i)$, for quadrature weights $v_{n,i}$ that do not depend on the function $g$ and that can be computed by a single weighted least squares procedure; and (c) it can be computed even when $f$ is known only up to a multiplicative constant. Point~(a) suggests that when the linear combinations of the functions $h_k$ span a rich function class, the integration error is likely to be small. Point~(b) implies that multiple integrals can be computed just as easily as a single one. Point~(c) shows that the approach is applicable for Bayesian computations. In addition, the control variates can be brought into play in a \emph{post-hoc} scheme, after generation of the particles and importance weights, and this for any AIS algorithm.

\textbf{Main result.} 
The main theoretical result of the paper is a probabilistic, non-asymptotic bound on  $\hat{\alpha}_n - \alpha$. Under appropriate conditions, the bound scales as $\tau / \sqrt n $, where $\tau^2$ is the scale constant in a sub-Gaussian tail condition on the error variable $\varepsilon = g - \alpha - \beta^\top h$ for $(\alpha, \beta) = \argmin_{a, b} \int (g - a - b^\top h)^2 f \, \diff \lambda $. Note that $\varepsilon$ has the smallest possible variance one could get using control variates $h$.  As a consequence, when the space of control variates is well suited for approximating $g$, the AISCV estimate will be highly accurate. Also, our bound depends only on the linear function space spanned by the control variates $h_1,\ldots,h_m$, not on the particular basis chosen in that space. The results rely on martingale theory, in particular on a concentration inequality for norm-subGaussian martingales in \cite{jin2019short}. In the course of the proof, we develop a novel bound on the smallest eigenvalue of certain random matrices, extending an inequality from \citep{Tropp2015} to the martingale case.
	
\textbf{Outline.}  
Section~\ref{sec:preliminaries} introduces the general framework of adaptive importance sampling and control variates. Next, Section~\ref{section:ais_and_cv} presents the AISCV estimate and the associated quadrature rule. Section~\ref{section:main_results} contains the statements of the theoretical results while Section~\ref{sec:practical} gathers practical considerations, including the construction of control variates. Numerical experiments are presented in Section~\ref{sec:numerical} and Section \ref{sec:conclu} concludes the main part of the paper with a discussion for further research. %Mathematical proofs and additional numerical details are available in the supplementary material.
	
	\section{Preliminaries on Monte Carlo integration}
	\label{sec:preliminaries}
	
	The aim of this section is to present the required mathematical framework for Monte Carlo integration and the variance reduction methods of interest, namely adaptive importance sampling and the control variate technique. 	Recall that $g : \reals^d \to \reals$ is an integrand and $f$ a probability density on $\reals^d$. The aim is to compute $\expec_f[g] = \int gf \, \diff \lambda$.

	\textbf{Adaptive importance sampling.} 
	In adaptive importance sampling (AIS), $\expec_f[g]$ is estimated by a weighted mean over a sample of random particles $X_1,\ldots,X_n$ in $\reals^d$.
	Since appropriate sampling densities naturally depend on $g$ and $f$, we generally cannot simulate from them. 
	They are then approximated in an adaptive manner by a family of tractable densities $(q_i)_{i \ge 0}$ that often evolve towards a density $q_{\mathrm{opt}}$ that optimizes some criterion. 
	%\js{Why the index $k$ rather than $i$? As you wish.}
	While the starting density $q_0$ is fixed, the density $q_i$ for $i \ge 1$ is determined in function of the particles $X_1,\ldots,X_i$ already sampled; think for instance of a parametric family, where the parameter of $q_i$ is a function of $X_1,\ldots,X_i$. Given the particles $X_1,\ldots,X_i$, the next particle, $X_{i+1}$, is then drawn from $q_i$.
	Formally, let $(X_i)_{i \ge 1}$ be a sequence of random vectors on $\reals^d$ defined on some probability space $(\Omega, \mathcal{F}, \prob)$.
	The distribution of the sequence $(X_i)_{i \ge 1}$ is specified by its policy as defined below.
	
	\begin{definition}[Policy]
		\label{def:policy}
		A policy is a random sequence of probability density functions $(q_i)_{i \ge 0}$ on $\reals^d$ adapted to the $\sigma$-field $(\mathcal{F}_i)_{i \ge 0}$ defined by $\mathcal{F}_0 = \{\varnothing, \Omega\}$ and $\mathcal{F}_i = \sigma(X_1,\dots,X_i)$ for $i \ge 1$. 
		The sequence  $(q_i)_{i \ge 0}$ is the policy of $(X_i)_{i \ge 1}$ whenever $X_i$ has density $q_{i-1}$ conditionally on $\mathcal{F}_{i-1}$.
	\end{definition}

	The (normalized) adaptive importance sampling estimate of $\expec_f[g]$ is then defined as
	\begin{equation}
	\label{eq:AIS}
		I^{\mathrm{(ais)}}_n(g) = \frac{\sum_{i=1}^n w_i g(X_i)}{\sum_{i=1}^n w_i}
		\quad \text{where} \quad w_i = \frac{f(X_i)}{q_{i-1}(X_i)} \quad \text{for}\ i=1,\ldots,n.
	\end{equation}
	The sampling weights $w_i$ reflect the fact that $X_i$ has been sampled from $q_{i-1}$ rather than from $f$.
	The division by $\sum_{i=1}^n w_i$ rather than by $n$ has two benefits: first, the integration is exact for constant integrands, and second, $f$ needs to be known only up to a proportionality constant, an advantage for Bayesian inference.
	
	Since updating the density $q_i$ at each iteration may be computationally expensive, it is customary to hold it fixed over a pre-determined number of iterations. Writing $n = n_1 + \cdots + n_T$ in terms of positive integers $(n_t)_{t=1}^T$ called the \emph{allocation policy}, the AIS estimate then becomes
	\begin{equation}
	\label{eq:AIS:nt}
		I^{\mathrm{(ais)}}_T(g) = \frac{\sum_{t=1}^T \sum_{i=1}^{n_t} w_{t,i} g(X_{t,i})}{\sum_{t=1}^T \sum_{i=1}^{n_t} w_{t,i}}
		\quad \text{where} \quad w_{t,i} = \frac{f(X_{t,i})}{q_{t}(X_{t,i})}
	\end{equation}
	for $t = 1,\ldots,T$ and $i=1,\ldots,n_t$. At stage $t$, the particles $X_{t,1},\ldots,X_{t,n_t}$ are sampled independently from $q_{t-1}$, while all particles sampled up to and including stage $t$ are used to determine the sampling density $q_t$ for stage $t+1$. It is easy to see that the two formulations of the AIS estimate are equivalent: \eqref{eq:AIS} arises from \eqref{eq:AIS:nt} by setting $n_t = 1$ for all $t$, while \eqref{eq:AIS:nt} can be obtained from \eqref{eq:AIS} by constructing the policy in such a way that the densities $q_i$ do not change within integer intervals of the form $\{0,\ldots,n_1-1\}$, $\{n_1,\ldots,n_1+n_2-1\}$, and so on. While the shorter representation \eqref{eq:AIS} is more convenient for theoretical purposes, formulation~\eqref{eq:AIS:nt} is the one used in practice (see Section~\ref{sec:numerical}).% for illustrations).

	Interestingly, the AIS estimate \eqref{eq:AIS} may be seen as a weighted least-squares estimate minimizing the loss function $a \mapsto \sum_{i=1}^n w_i \left(g(X_i) - a\right)^2$.
	This perspective is key to understand control variates.
	
	\textbf{Control variates.} The control variates method is a variance reduction technique that consists in incorporating a new piece of information---the known values of the integrals of some control functions---in a basic Monte Carlo framework.
	Control variates are simply functions $h_1,\ldots,h_m \in L_2(f)$ with known integrals.
	Without loss of generality, assume that $\expec_f[h_j]=0$ for all $j=1,\ldots,m$.
	Let $h = (h_1,\ldots,h_m)^\top$ denote the $\reals^m$-valued function with the $m$ control variates as elements. 
	For any coefficient vector $\beta \in \reals^m$, we have $\expec_f[g - \beta^\top h] = \expec_f[g]$. Given an independent random sample $X_1,\ldots,X_n$ from $f$, any $\beta \in \reals^m$ therefore results in an unbiased estimator of $\expec_f[g]$ by
	\begin{equation}
	\label{eq:CV}
		I_n^{(\mathrm{cv})}(g,\beta) = \frac{1}{n} \sum_{i=1}^n \left( g(X_i) - \beta^\top h(X_i) \right).
	\end{equation}
%	The variance of $I_n^{(\mathrm{cv})}(g,\beta)$ is minimal for $\beta$ equal to
%	\begin{align*}
%		\beta^\star \in \argmin_{\beta \in \rset^m} \expec_f \left[\left(g-\expec_f[g]-\beta^\top h\right)^2\right],
%		% = \argmin_{\beta \in \reals^m} \sigma^2_m(g).
%	\end{align*}
%	with minimal variance
%	\[
%		\sigma_m^2(g) = \min_{\beta \in \rset^m} \expec_f \left[\left(g-\expec_f[g]-\beta^\top h\right)^2\right],
%	\]
%	Since $(\beta^\star)^\top h$ is the $L_2(f)$-projection of 
%	$g - \expec_f[g]$ on the linear space generated by the control variates $h_1,\ldots,h_m$, the Hilbert projection theorem ensures the existence and uniqueness of the optimal linear combination $(\beta^*)^\top h$. 
	Provided the $m \times m$ covariance matrix $G = \E_f[h h^\top]$ is invertible, there is a unique coefficient vector $\beta^* \in \reals^m$ for which the variance of $I_n^{\mathrm{(cv)}}(g)$ is minimal and it is given by
	\begin{equation}
	\label{eq:beta*}
		\beta^* = \left(\E_f[h h^\top]\right)^{-1} \E_f[h g].
	\end{equation}
	This vector being generally unknown, it needs to be estimated from the particles $X_1,\ldots,X_n$. Casting the problem in an ordinary least squares framework leads to the control variate estimate
	\begin{equation}
	\label{eq:OLS}
	\begin{split}
		I_n^{\mathrm{(cv)}}(g) 
		&= I_n^{\mathrm{(cv)}} \bigl(g, \hat{\beta}_n^{\mathrm{(cv)}}\bigr)
		= \hat{\alpha}_n^{\mathrm{(cv)}}
		\qquad \text{where} \\
		\bigl( 
			\hat{\alpha}_n^{\mathrm{(cv)}}, \hat{\beta}_n^{\mathrm{(cv)}} 
		\bigr)
		&\in \argmin_{(a, b) \in \reals \times \reals^m} \frac{1}{n} \sum_{i=1}^n \left( g(X_i) - a - b^\top h(X_i) \right)^2.
	\end{split}
	\end{equation}
	The estimator $I_n^{\mathrm{(cv)}}(g)$ is well-defined provided the minimizer $\hat{\alpha}_n^{\mathrm{(cv)}}$ to \eqref{eq:OLS} is unique. This is the case if and only if there does not exist $b \in \reals^m$ such that $b^\top h(X_i) = 1$ for all $i = 1,\ldots,n$.
		
	The asymptotic distribution of $I_n^{\mathrm{(cv)}}(g)$ as $n \to \infty$ is the same as if the variance-minimizing vector $\beta^*$ were used in \eqref{eq:CV}. In particular, the asymptotic variance of $I_n^{\mathrm{(cv)}}(g)$ is $\sigma_m^2(g) / n$ where
	\[
		\sigma_m^2(g) = \min_{\beta \in \rset^m} 
		\expec_f \bigl[ (g-\expec_f[g]-\beta^\top h)^2 \bigr].
	\]
	Interestingly, when using only the first $\ell$ out of $m$ control variates, where $\ell \in \{0, 1, \ldots,m\}$, we have $\sigma_m^2(g) \le \sigma_\ell^2(g)$. In terms of asymptotic variance, it therefore never harms to add more control variates. Their construction will be addressed in Section~\ref{subsec:cv_practice}. 
	
	\iffalse
	\begin{remark}(Choice of control variates) When the target density $f$ is proportional to well-know densities, one way to build control variates is to rely on families of orthogonal polynomials. Furthermore, when one has acces to the derivatives of the log-target $\nabla_{x} \log(f)$, Stein's method allows to build infinitely many control variates \citep{oates2017control}. This setting often arises in Bayesian inference where the target density $f$ is equal to the posterior distribution and the derivatives of interest can be easily computed through the log-likelihood function. Both cases are considered in the numerical experiments with further details.
	\end{remark}
	
	\begin{remark}(Number of control variates) Interestingly, when using $l$ control variates with $0 \leq l \leq m$, we have $\sigma^2_m(g) \leq \sigma^2_l(g)$ so the use of control variates always reduce the variance of the Monte Carlo estimate.
	\end{remark}
	\fi
	
\section{Combining adaptive importance sampling with control variates} \label{section:ais_and_cv}

\textbf{AISCV estimator.}
Consider the same integration problem $\expec_f[g] = \int gf \, \diff \lambda$ as in Section~\ref{sec:preliminaries}. With the idea of performing variance reduction when calculating integrals with respect to the posterior density in Bayesian inference, we incorporate control variates into the AIS estimate.
Let the particles $(X_i)_{i \ge 1}$ be generated according to a policy $(q_i)_{i \ge 0}$ as in Definition~\ref{def:policy}. Let $h = (h_1,\ldots,h_m)^\top$ be a vector of control variates, i.e., $h_j \in L_2(f)$ and $\expec_f[h_j] = 0$ for every $j = 1,\ldots,m$. Combining \eqref{eq:AIS} and \eqref{eq:CV}, the proposed estimate takes the form
\begin{equation}
\label{eq:Iaiscvgbetanorm}
	I^{\mathrm{(aiscv)}}_n(g, \beta) 
	= \frac{\sum_{i=1}^{n} w_i \left(g(X_i) - \beta^\top h(X_i)\right)}{\sum_{i=1}^{n} w_i}, 
\end{equation}
where $\beta \in \reals^m$ remains to be determined. To do so, the ordinary least-squares problem in \eqref{eq:OLS} is replaced by a weighted one, yielding the novel AISCV estimator
\begin{equation}
\label{eq:AISCV}
	\begin{split}
	I_n^{\mathrm{(aiscv)}}(g) 
	&= I_n^{\mathrm{(aiscv)}} \bigl(g, \hat{\beta}_n\bigr)
	= \hat{\alpha}_n \qquad \text{where} \\
	\bigl( \hat{\alpha}_n, \hat{\beta}_n \bigr)
	&\in \argmin_{(a, b) \in \reals \times \reals^m} \sum_{i=1}^n w_i \left( g(X_i) - a - b^\top h(X_i) \right)^2.
	\end{split}
\end{equation}
The estimator is well-defined only if the minimizer $\hat{\alpha}_n$ is unique---the minimizer $\hat{\beta}_n$ need not be. We will come back to this in the next paragraph.

As in \eqref{eq:AIS:nt}, the policy may be divided into $T$ stages in order to reduce the number of times the sampler needs to be updated. Stage $t = 1,\ldots,T$ has length $n_t$, with $\sum_{t=1}^T n_t = n$. Within each stage, the sampling density remains constant. In practice, this leads to the AISCV estimate in Algorithm~\ref{algo:aiscv}.

	\begin{algorithm}[h]
	\setstretch{1.25}
	\caption{Adaptive Importance Sampling with Control Variates (AISCV)}\label{alg:AISCV}
	\begin{algorithmic}[1]
		\Require integrand $g$, target density $f$ (up to a proportionality constant), number of stages $T \in \naturals^*$, allocation policy $(n_t)_{t=1}^T$, initial density $q_0$, update rule for the sampling policy
		
		\For{$t = 1,\dots,T$}
		\State Generate an independent random sample $X_{t,1},\dots,X_{t,n_t}$ from $q_{t-1}$ 
		\State Compute the vector of weights $(w_{t,i})_{i=1}^{n_t}$ where $w_{t,i} = f(X_{t,i})/q_{t-1}(X_{t,i})$
		\State Construct the matrix of control variates $H_t = \bigl( h_j(X_{t,i}) \bigr)_{i=1,\dots,n_t}^{j=1,\dots,m}$
		\State Evaluate the integrand in the particles: $ (g(X_{t,i}))_{i=1}^{n_t}$
%		\State Store the weights $W_t$, the matrix $H_t$ and $g_t=(g(X_{t,1}),\dots, g(X_{t,n_t}))^\top$
		\State Update the sampler $q_t$ based on all previous particles $(X_{s,i} : s = 1,\ldots,t; i = 1,\ldots,n_s)$
		\EndFor
		\State Compute $(\hat{\alpha}_T, \hat{\beta}_T) 
		= \argmin_{(a,b) \in \reals \times \reals^m} \left\{ \sum_{t=1}^{T} \sum_{i=1}^{n_t} w_{t,i} \left(g(X_{t,i}) - a - b^\top h(X_{t,i})\right)^2 \right\}$
		\State\Return $I_n^{\mathrm{(aiscv)}}(g) = \hat{\alpha}_T$.
	\end{algorithmic}
	\label{algo:aiscv}
\end{algorithm}

\iffalse
	\textbf{AISCV estimator.} In the idea of performing variance reduction for general Bayesian inference, we consider the adaptive importance sampling Monte Carlo estimator defined in \eqref{eq:def_ais} with the addition of control variates. The resulting AISCV estimator is defined by
	\begin{equation}
		\label{eq:Iaiscvgbeta}
		I_n^{\mathrm{(aiscv)}}(g, \beta)= \frac{1}{n} \sum_{i=1}^{n} w_i (g(X_i) - \beta^\top h(X_i)), \quad X_i \sim q_{i-1},
	\end{equation}
	and its normalized counterpart is given by 
	\begin{equation}
		\label{eq:Iaiscvgbetanorm}
		I^{\mathrm{(aiscv)}}_{\mathrm{norm}}(g, \beta)= \left(\sum_{i=1}^{n} w_i (g(X_i) - \beta^\top h(X_i)) \right) / \left(\sum_{i=1}^{n} w_i \right). 
	\end{equation}
\fi

\textbf{Quadrature rule.} 
The AIS estimate \eqref{eq:AIS} is a quadrature rule with quadrature points $X_i$ and quadrature weights proportional to the sampling weights $w_i$. The AISCV estimate \eqref{eq:AISCV} has the same property, but with adapted quadrature weights.
Let $e_n = (e_{n,i})_{i=1,\ldots,n}$ be the vector of residuals resulting from the weighted least-squares regression of the constant vector $\one_n = (1, \ldots, 1)^\top \in \reals^n$ on the control variates but without intercept:
\begin{equation}
\label{eq:eni}
\begin{split}
e_{n,i} 
&= 1 - \hat{\beta}_n(\one_n)^\top h(X_i) \qquad \text{where} \\
\hat{\beta}_n(\one_n)
&\in \argmin_{b \in \reals^m} \sum_{i=1}^n w_i \left(1 - b^\top h(X_i)\right)^2.
\end{split}
\end{equation}
Even though the vector $\hat{\beta}_n(\one_n)$ is not necessarily unique, the weighted least squares fit $(\hat{\beta}_n(\one_n)^\top h(X_i))_{i=1,\ldots,n}$ always is.
According to the next proposition, the quadrature weights are proportional to $(w_i e_{n,i})_{i=1,\ldots,n}$.

\begin{proposition}[AISCV quadrature rule]
	\label{prop:AISCV:quad}
	The minimizer $\hat{\alpha}_n$ in \eqref{eq:AISCV} is unique if and only if $e_n \ne 0$ in \eqref{eq:eni}. In that case, the AISCV estimate is
	\begin{equation}
	\label{eq:aiscv:quad}
	I_n^{\mathrm{(aiscv)}}(g) =
	\hat{\alpha}_n 
	= \frac{\sum_{i=1}^n w_i e_{n,i} g(X_i)}{\sum_{i=1}^n w_i e_{n,i}}.
	\end{equation}
\end{proposition}

If $e_n = 0$, then there exists $b \in \reals^m$ such that $b^\top h(X_i) = 1$ for all $i = 1,\ldots,n$. In that case, the minimizer $\hat{\alpha}_n$ in \eqref{eq:AISCV} is not unique and the AISCV estimate is not well-defined. To remedy this, one can for instance reduce the number of control variates. This issue already occurs with the ordinary control variate estimator in \eqref{eq:CV}.

Rather than requiring a different weighted least squares problem for every integrand $g$ as in \eqref{eq:AISCV}, the quadrature rule in \eqref{eq:aiscv:quad} only involves a single weighted least squares problem \eqref{eq:eni}, whatever $g$. 
Given the quadrature weights, calculating the AISCV estimate for a novel integrand only requires the evaluations of that function on the sampled particles, making the whole procedure a \textit{post-hoc} scheme. The steps in case the sampling policy is divided into $T$ stages are given in Algorithm~\ref{alg:AISCV:quad}, which gives the same result as Algorithm~\ref{algo:aiscv}, but with less effort if multiple integrands $g$ are into play.

\begin{algorithm}
	\setstretch{1.25}
	\caption{Quadrature Rule -- AISCV \textit{post-hoc} scheme}\label{alg:AISCV:quad}
	\begin{algorithmic}[1]
		\Require integrand $g$, $T \in \naturals^*$, allocation policy $(n_t)_{t=1}^T$, weights $(w_t)_{t=1}^{T}$ with $w_t = (w_{t,i})_{i=1}^{n_t}$, matrices $(H_t)_{t=1}^T$ with $H_t = \bigl(h_j(X_{t,i})\bigr)_{i=1,\ldots,n_t}^{j=1,\ldots,m}$, particles $(X_{t,i} : t=1,\ldots,T; i=1,\ldots,n_t)$	
		\State Compute $\hat{\beta}_n(\one_n) = \argmin_{b \in \reals^m} \sum_{t=1}^T \sum_{i=1}^{n_t} w_{t,i} \left( 1 - b^\top h(X_{t,i})\right)^2 $ 
		\State Compute $u_t = \diag(w_t) [\one_{n_t} - H_t \hat{\beta}_n(\one_n)]$ for $t = 1,\ldots,T$
		\State Compute $s = \sum_{t=1}^T \sum_{i=1}^{n_t} u_{t,i}$
		\State Compute weights $v_{t,i} = u_{t,i} / s$ for $t = 1,\ldots,T$ and $i=1,\ldots,n_t$
		\State\Return $I_T^{\mathrm{(aiscv)}}(g) = \sum_{t=1}^T \sum_{i=1}^{n_t} v_{t,i} g(X_{t,i})$
	\end{algorithmic}
\end{algorithm}

\section{Theoretical properties of the AISCV estimate}
\label{section:main_results}

Here we point out several theoretical properties of the novel AISCV estimate. A first point is that the integration rule is exact on the linear span of the control variates and the constant function.

\begin{proposition}[Exact integration] 
	\label{prop:exact}
	For integrands of the form $g = \alpha + \beta^\top h$ for $\alpha \in \reals$ and $\beta \in \reals^m$, the AISCV estimate is exact: $I^{\mathrm{(aiscv)}}_n(g) = \alpha = \expec_f[g].$
	%	Let $\mathcal{H}_m = Span(h_1,\ldots,h_m)$ denote the closed linear subspace generated by the control variates. The normalized AISCV estimator satisfies the following exact integration properties: \\
	%	\textit{(i)} if the integrand $g$ is constant then $I^{\mathrm{(aiscv)}}_{\mathrm{norm}}(g)=I(g)$. \\
	%	\textit{(ii)} if the integrand $g \in \mathcal{H}_m$ is a linear combination of control variates then $I^{\mathrm{(aiscv)}}_{\mathrm{norm}}(g)=I(g)$.
	%	\label{prop:exact_integration}
\end{proposition}

A second property is that we may apply arbitrary invertible linear transformations to the control variates without changing the AISCV estimate. This can be advantageous computationally, to make the underlying weighted least squares problem more stable numerically. Also, it means that without loss of generality, we may assume that the control variates are uncorrelated and have unit variance, which simplifies the theoretical performance analysis.

\begin{proposition}[Invariance]
\label{prop:invariance}
	If the matrix $A \in \reals^{m \times m}$ is invertible, then the AISCV estimate based on the control variates $A h$ is the same as the one based on $h$.
\end{proposition}

%\begin{remark}
%	The two previous remarks do not hold for the unnormalized AISCV estimate $I^{\mathrm{(AISCV)}}(g, \hat{\beta}_n)$. 
%\end{remark}

%\textbf{Non-asymptotic bound for AISCV estimator.} 

%We now provide a non-asymptotic bound %, valid with high probability, 
Our main result is a non-asymptotic bound on the error of the AISCV estimate for $\int gf \, \diff \lambda$ when $\int g^2 f \, \diff \lambda $ is finite. First, we introduce some assumptions and definitions. % needed to derive the error bound. 

The first condition that is required concerns the policy given by the AIS part of the algorithm. It is supposed that any element from the policy should dominate the function $f$.

	\begin{assumption}[Dominated measures] \label{as:bound_of_density}
		There exists $c \ge 1$ such that, for all $x \in \reals^d$ and for any $i=1,\dots,n$, we have $f(x) \leq c \cdot q_i(x)$.
			\end{assumption}	

This assumption represents a \textit{safe} approach to importance sampling, as the policy will always allow to sample in places where $f$ is positive. A well-known and well-spread \citep{hesterberg1995weighted,owen+z:2000,delyon2021safe} technique to achieve such a defensive strategy is to a use mixture density $q_i = (1-\eta) f_i +  \eta q_0 $ where $\eta \in (0,1)$ and where $q_0$ has sufficiently heavy tails to dominate $f$. Such a mixture allows to choose the densities $f_i$ with some flexibility using in principle any AIS algorithm. 
Second, the control variates shall be linearly independent and bounded.

	\begin{assumption}[Control variates] \label{as:bounded_cv}
		We have $\sup_{x : f(x) > 0} | h_j(x) | < \infty$ for all $j=1,\ldots,m$. The matrix $G = \int hh^\top f \, \diff \lambda $ is invertible.
	\end{assumption}

The previous condition allows to define the standardized vector of control variates as $\hbar  = G^{-1/2} h$. By Proposition~\ref{prop:invariance}, this change does not affect the AISCV estimate. The orthonormal control variates $\hbar$ will play a key role through the following quantity $$	B = \sup_{x : f(x) > 0} \|\hbar (x) \|_2^2.$$
The quadratic form $\|\hbar (x) \|_2 ^2 = h(x)^\top G^{-1} h(x)$ is referred to as the \textit{leverage function} in ordinary linear regression as it quantifies the influence of a training point $x$ on the prediction of the observed response. It is invariant with respect to invertible linear transformations of the control variate vector.

Assumption~\ref{as:bounded_cv} and the fact that the integrand $g$ is square integrable with respect to $f$ allows to define the residual function $\varepsilon = g -  \int g f \, \diff \lambda  - h^\top \beta^*$ where $\beta^*$ has been introduced in \eqref{eq:beta*} as a minimizer of the residual variance. Since we work in the space $L^2(f)$, we assume without loss of generality that $g$ and $h$ vanish outside $\{x : f(x) > 0\}$ and we put $\varepsilon(x) = 0$ for $x \in \reals^d$ such that $f(x) = 0$. %\js{D'accord?}
The residual function $\varepsilon$ should satisfy the following tail condition.

\begin{assumption}[Residual tail]
	\label{as:sub_gauss}
	There exists $\tau > 0$ such that, for all $t > 0$ and all integer $i \ge 1$, we have $\prob [ | w_i \varepsilon(X_i) | >t  \mid \mathcal F_{i-1} ] \leq 2  \exp( -  t  ^2 / ( 2\tau^2)  )$.
\end{assumption}

The previous assumption concerns both the function $\varepsilon$ and the policy sequence $(q_i)_{i \ge 0}$. %Under Assumption \ref{as:bound_of_density}, the previous assumption is implied by the following tail inequality on $\varepsilon(X_i)$: $\prob [ | \varepsilon(X_i) | >t  \mid \mathcal F_{i-1} ] \leq  2  \exp( -  t  ^2 / ( 2\tau^2)  )$ for all $t>0$. 
 Since $\mathbb E [  w_i \varepsilon(X_i) \mid \mathcal F_{i-1} ] = 0$, it is implied by the so-called sub-Gaussian condition \citep{boucheron2013concentration} that $ \mathbb E [ \exp( \lambda  w_i \varepsilon( X_i ) )  \mid \mathcal F_{i-1} ] \leq \exp( - \lambda  ^2 \tau^2 /2 ) $ for any $\lambda \in \reals$.
In the proof of Theorem~\ref{th:concentration_inequality}, Assumption~\ref{as:sub_gauss} allows to derive concentration bounds on residual-based sums using recent results from  \cite{jin2019short,LelucPortierSegers2021}. We are now in position to state our main result on the error of the AISCV estimate. %whsoe proof is given in the supplementary material.
%\begin{assumption}[Residual tail]
%	\label{as:sub_gauss}
%	There exists $\tau > 0$ such that, for all $t > 0$ and all integer $i \ge 1$, we have $\prob [ | \varepsilon(X_i) | >t  \mid \mathcal F_{i-1} ] \leq 2  \exp( -  t  ^2 / ( 2\tau^2)  )$.
%\end{assumption}

%Since $\prob [ | \varepsilon(X_i) | >t  \mid \mathcal F_{i-1} ] =  \int_{| \varepsilon(x) | >t} q_{i-1} (x)  \, \diff x $, the previous assumption concerns both the function $\varepsilon$ and the policy sequence $(q_i)_{i \ge 0}$.
%It is implied by the so-called sub-Gaussian condition \citep{boucheron2013concentration} that $\int \exp( \lambda \varepsilon( x ) )  q_{i-1}(x)  \, \diff x  \leq \exp( - \lambda  ^2 \tau^2 /2 ) $ for any $\lambda \in \reals$. \js{The latter assumption implies that $\int \varepsilon(x) q_{i-1}(x) \, \diff x = 0$, which is not the case in general---we rather have $\int w_i(x) \varepsilon(x) q_{i-1}(x) \, \diff x = 0$. In general, the sub-Gaussian condition on $w_i \varepsilon$ w.r.t.\ $q_{i-1}$ does not imply one on $\varepsilon$, I believe. I propose to delete the previous sentence and not to include the proposed remark on the meaning of $\tau$ (François' email a few days ago; commented out in LaTeX below) either. But feel free to do otherwise, just make sure everything's correct.} 
%In the proof of Theorem~\ref{th:concentration_inequality}, Assumption~\ref{as:sub_gauss} will allow to derive concentration bounds on residual-based sums using recent results from  \cite{jin2019short,LelucPortierSegers2021}. 
%See Remark~\ref{rk:tau} on the nature of $\tau$.

	\begin{theorem}[Concentration inequality for AISCV estimate] \label{th:concentration_inequality}
		If Assumptions \ref{as:bound_of_density}, \ref{as:bounded_cv} and \ref{as:sub_gauss} hold, then, for any $\delta \in (0, 1)$ and for all $n \geq  C_1  c^2   B  \log(10m/\delta)$, we have, with probability at least $1-\delta$, that
		\[
		\left| 
		I^{\mathrm{(aiscv)}}_{n}(g)
		- \int_{\reals^d} g(x) f(x) \, \diff x
		\right|
		\le 
		C_2    \tau  \sqrt{ \frac{ \log(10/\delta)  } {  n}  } 
		+  C_3 c B \tau  \frac{ \log(10m/\delta) }{  n},
		\]
		where $C_1$, $C_2$, $C_2$ are universal constants specified in the proof.
		%\js{I've put "universal", i.e., the constants don't depend on anything, they're just numbers, like 365. Correct?} 
	\end{theorem}

\begin{remark}[Understanding $\tau$]
	\label{rk:tau}
	%Because $ \mathbb E [ w_i\varepsilon(X_i) \mid \mathcal F_{i-1} ]  = 0 $,
	The quantity $\tau$ in Assumption~\ref{as:sub_gauss} is related to the conditional variance $ \mathbb E [ w_i^2 \varepsilon^2(X_i) \mid \mathcal F_{i-1} ]$. %The two ($\tau$ and the variance) 
	They actually coincide when $ w_i \varepsilon(X_i)$ is Gaussian. For a policy satisfying Assumption \ref{as:bound_of_density}, $ \mathbb E [ w_i^2 \varepsilon^2(X_i) \mid \mathcal F_{i-1} ]\leq c \sigma_m^2 $  which for certain combinations of integrands and control functions scales as $m^{-s/d}$ \cite{PortierSegers2019} where the parameter $s$ represents the degree of smoothness of $g$. 
\end{remark}

\begin{remark}[Convergence rates]
	Consider an asymptotic regime where the number of control variates $m$ tends to infinity with the sample size $n$. %In the asymptotic regime where $m,n \to \infty$, t
The AISCV estimate improves upon the AIS method $(m=0)$, which has rate $1 / \sqrt{n}$, as soon as $\tau+\tau B \log (m) / \sqrt{n} \rightarrow 0$. To recover the same order of an oracle estimate with rate $\tau / \sqrt{n}$, one must have $B \log (m)= O(\sqrt{n})$ %\js{$O(n)$ replaced by $O(\sqrt{n})$; correct?} 
	as $n \rightarrow \infty$.%, that is, $m$ must not be too large compared to $n$.
\end{remark}

\section{Practical considerations} 
\label{sec:practical}

This section presents ways to build control variates using either families of polynomials or general functions based on Stein's method, with a highlight on computations in the Bayesian framework.
%This section presents several ways to build control variates %from a practical point of view 
%using either families of polynomials or general functions based on Stein's method. Next, some computations are highlighted in the framework of Bayesian inference.
	\subsection{Control variate constructions} \label{subsec:cv_practice}
	
	\textbf{Orthogonal polynomials.} 
	When the target density $f$ can be decomposed as a product of univariate densities $f = p_1 \otimes \cdots \otimes p_d$, multidimensional control functions may be constructed based on univariate ones. This happens for instance for the uniform distribution over the unit cube $[0,1]^d$ or with uncorrelated Gaussian distributions on $\rset^d$. Such univariate control variates may be easily constructed using families of polynomials \citep{gautschi2004orthogonal}, such as Legendre polynomials for the uniform distribution on $[0,1]$ and Hermite polynomials for the Gaussian distribution on $\rset$. This technique can also be used when $f$ is dominated by another density $f^*$ having the said product form by transforming zero-mean control variates $h^*$ with respect to $f^*$ via $h = h^* f^* / f$. %\js{Correct? Clear? Relevant?}
	
	Let $(h_{1}, \ldots, h_{k})$ be a vector of univariate control functions with respect to a density $p$, i.e., $\expec_p[h_j]=0$ for all $j=1,\ldots,k$. Let $h_0=1$ denote the constant function equal to one. For a multi-index $\ell = (\ell_{1}, \ldots, \ell_{d})$ in $\{0,\ldots, k\}^{d} \setminus \{(0, \ldots, 0)\}$,  multivariate controls with respect to $p^{\otimes d}$ are built by forming tensor products of the form $ h_{\ell}(x_{1}, \ldots, x_{d})= h_{\ell_{1}}(x_{1}) \cdots  h_{\ell_{d}}(x_{d})$ , yielding a total number of $m=(k+1)^{d}-1$ control functions. Alternative approaches yielding smaller control spaces consist of imposing $\ell_{j} = 0$ for all but a small number %(one or two, say) 
	of coordinates $j = 1,\ldots,d$ or by the constraint $\ell_1 + \cdots + \ell_d \leq Q$ for some $Q \geq 1$. 
	
	%bounding the total degree of the multivariate polynomial via
	\textbf{Stein control variates.} In the general case where one has only access to the evaluations of $f$, control variates may be constructed using Stein's method.
	The technique relies on the gradient $\nabla_{x} \log f(x)$ which can either be directly computed (see the example of Bayesian regression below) or which may be available through automatic differentiation provided in popular API's such as Tensorflow and PyTorch \citep{abadi2016tensorflow,paszke2017automatic}. 
	Let $\Delta_x = \nabla_{x}^\top \nabla_{x}$ denote the Laplace operator. By definition, the second-order Stein operator $\mathcal{L}$ \citep{stein1972bound,gorham2015measuring} associated to the density $f$ is defined by:	
\begin{align*}
\forall \varphi \in \mathcal{C}^2(\rset^d,\rset), \quad 
(\mathcal{L}\varphi)(x) = \Delta_{x} \varphi (x) + \nabla_{x} \varphi (x)^\top \nabla_{x} \log f(x).
\end{align*} 
The transformation guarantees that $\expec_f[\mathcal{L}\varphi]=0$ for all $\varphi$ with weak regularity conditions \citep{mira2013zero}. %belonging to a Stein functions class. 
	Therefore, we can build infinitely many control variates $h_{\varphi} = \mathcal{L}\varphi$ from given functions $\varphi$. One simple way is to let $\varphi$ be a polynomial with bounded total degree: for a degree vector $\bm{\alpha} = (\alpha_1,\ldots,\alpha_d) \in \mathbb{N}^d$ with $\alpha_1+\cdots+\alpha_d \le Q$, define $\varphi_{\bm{\alpha}}(x) = x_1^{\alpha_1} \cdots x_d^{\alpha_d}$.
	%Note that in Bayesian inference, $f$ is proportional to the likelihood and is general models (regression and classification), the closed-form for the derivative of the log-likelihood $\nabla_{x} \log \mathcal{L}(X|x) $ is available (and easy!).
	Given the dimension $d$ and the total degree $Q$, there are $m = \binom{d+Q}{d} - 1$ such degree vectors, yielding the associated control variates $h_{\bm{\alpha}} = h_{\varphi_{\bm{\alpha}}}$. %The Stein operator requires the evaluations of the gradient and Laplace operator. % $\nabla_x \varphi$ and $\Delta_x \varphi$. %For the considered polynomials, the gradient is $\nabla_{x} \varphi_{\bm{\alpha}} = (\partial \varphi_{\bm{\alpha}}/\partial x_1, \ldots, \partial \varphi_{\bm{\alpha}}/\partial x_d)^T$ and the laplacian is $\Delta_{x} \varphi_{\bm{\alpha}} = \sum_{k=1}^d (\partial^2 \varphi_{\bm{\alpha}}/\partial x_k^2)$ where
	%\begin{align*}
	%   \frac{\partial \varphi_{\bm{\alpha}}}{\partial x_k}= \alpha_k x_k^{\alpha_k - 1}\prod_{j \neq k} x_j^{\alpha_j}, 
	%  \qquad \frac{\partial^2 \varphi_{\bm{\alpha}}}{\partial x_k^2}=\alpha_k (\alpha_k-1) x_k^{\alpha_k - 2}  \prod_{j \neq k} x_j^{\alpha_j}
	%\end{align*}
	For fast computation, note that, writing $\phi_{\bm{\alpha}}(x) =  \varphi_{\bm{\alpha}}(x) \mathds{1}_d$, $D_1(x) = \diag(\alpha_1/x_{1},\ldots,\alpha_d/x_{d})$ and $D_2(x) = \diag(\alpha_1(\alpha_1-1)/x_{1}^2,\ldots,\alpha_d(\alpha_d-1)/x_{d}^2)$, we have $\nabla_{x} \varphi_{\bm{\alpha}}(x) = D_1(x)\phi_{\bm{\alpha}}(x)$ and $\Delta_{x} \varphi_{\bm{\alpha}}(x) = \mathds{1}_d^\top (D_2(x)\phi_{\bm{\alpha}}(x)).$ %From a practical perspective, it is enough to first build 
	In practice, all combinations of $\bm{\alpha}$ are stored in a matrix $A \in \mathbb{N}^{m \times d}$. 
%	Then, at time instant $i$, for the particle $X_i \sim q_{i-1}$ and the degree vector $\bm{\alpha} = (A_{i j})_{j=1}^d$, the control variates are computed using $h_{\bm{\alpha}}(X_i) = \Delta_{x} \varphi_{\bm{\alpha}}(X_i) + \nabla_{x} \varphi_{\bm{\alpha} }(X_i)^\top \nabla_{x} \log f(X_i).$
	
\subsection{Bayesian inference}
\label{sec:bayes}

Given data $\mathcal{D}$ and a parameter of interest $\theta \in \Theta \subset \rset^d$, posterior integrals take the form $\int_{\rset^d} g(\theta) p(\theta|\mathcal{D}) \, \diff\theta,$ where $p(\theta|\mathcal{D}) \propto \ell(\mathcal{D}|\theta)\pi(\theta)$ is the posterior distribution, proportional to a prior $\pi(\theta)$ and a likelihood function $\ell(\mathcal{D}|\theta)$. For instance, when $g(\theta)=\theta$, the integral above recovers the posterior mean. Stein control variates involve the computation of the gradient of the log-posterior $\nabla_{\theta} \log p(\theta|\mathcal{D})$, which implicitly relies on the score function $\nabla_{\theta} \log \ell(\mathcal{D}|\theta)$. We point out two common examples---linear and logistic regression---where these functions are easy to compute.
	
\textbf{Bayesian linear regression.}
Consider a linear regression problem comprised of observations $X \in \rset^{N \times d}$ with labels $y \in \rset^N$. In the Gaussian fixed design setting, the predictor $x_i$ produces the response $y_i = x_i^\top \theta + \varepsilon_i$ where $\varepsilon_1,\ldots,\varepsilon_N \sim \mathcal{N}(0,\sigma^2)$ are centered Gaussian noises. %are independent zero-mean Gaussian random variables with have common variance $\sigma^2 > 0$, assumed to be known for simplicity. 
The likelihood $\ell(X,y|\theta)$ is proportional to $(\sigma^2)^{-N/2} \exp(-(y-X \theta)^\top (y-X \theta)/(2 \sigma^2))$, yielding the score function $\nabla_{\theta} \log \ell(X,y|\theta) = X^\top(y - X\theta) /(2 \sigma^2)$.
	
\textbf{Bayesian logistic regression.}
Next, consider the logistic regression problem comprised of observations $X \in \rset^{N \times d}$ with associated binary labels $y \in \{0,1\}^N$. Letting $\sigma(s) = 1 / (1 + e^{-s})$ denote the sigmoid function, the likelihood function is  $\ell(X,y|\theta) = \prod_{i=1}^N \sigma(\theta^\top x_i)^{y_i} (1-\sigma(\theta^\top x_i))^{1-y_i}$. The score function is simply $\nabla_{\theta} \log \ell(X,y|\theta) = X^\top(y - \sigma(X\theta))$.
	
\section{Numerical illustration}
\label{sec:numerical}

To compare the finite-sample performance of the AIS and AISCV estimators, we first present in Section~\ref{ref:synthetic} synthetic data examples involving the integration problem over the unit cube $[0, 1]^d$ and then with respect to some Gaussian mixtures as in \cite{cappe2008adaptive}. The goal is to compute $\int g f \, \diff \lambda$ for vectors of integrands $g:\rset^d \to \rset^p$. We consider various dimensions $d > 1$ and several choices for the number of control variates $m$. Section \ref{sec:real_world} deals with real-world datasets in the context of Bayesian inference. For ease of reproducibility, the code, numerical details and additional results are available in the supplementary material.
	
\textbf{Parameters.} In all simulations, the sampling policy is taken within the family of multivariate Student $t$ distributions of degree $\nu$ denoted by $\{q_{\mu,\Sigma_0}\,:\, \mu \in \reals^d \}$ with $\Sigma_0 =\sigma_0 I_d(\nu-2)  /\nu$ and $\nu >2, \sigma_0 >0$.
Similarly to \cite{portier2018asymptotic}, the mean $\mu_t$ is updated at each stage $t=1,\ldots, T$ by the generalized method of moments (GMM), leading to $\mu_t	= (\sum_{s=1}^t \sum_{i=1}^{n_s} w_{s,i }X_{s,i}) /  ({\sum_{s=1}^t \sum_{i=1}^{n_s} w_{s,i}})$. The allocation policy is fixed to $n_t = \num{1000}$ and the number of stages is $T \in \{5;10;20;30;50\}$. The different Monte Carlo estimates are compared by their mean squared error (MSE) obtained over $100$ independent replications.
	
	\subsection{Synthetic examples} \label{ref:synthetic}
	
\textbf{Integration on $[0,1]^d$.} We seek to integrate functions $g$ with respect to the uniform density $f(x) = 1$ for $x \in [0, 1]^d$ in dimensions $d \in \{4;8\}$. We rely on Legendre polynomials for the control variates. Consider the integrands $g_1(x) = 1 + \sin(\pi(2 d^{-1} \sum_{i=1}^d x_{i} - 1))$, $g_2(x) = \prod_{i=1}^d (2/\pi)^{1/2} x_i^{-1} \mathrm{e}^{-\log(x_i)^2/2}$ and  $g_3(x) = \prod_{i=1}^d \log(2) 2^{1-x_i}$, all of which integrate to $1$ on $[0, 1]^d$. 
%The functions $g_2$ and $g_3$ are built using tensor products of log-normal and exponential density functions, respectively. 
None of the integrands is a linear combination of the control variates. The policy parameters are $\mu_0 = (0.5,\ldots, 0.5)\in \reals^d$, $\nu = 8$, and $\sigma_0 = 0.1$. The control variates are built out of tensor products of Legendre polynomials where the degree $\ell_j$ equals $0$ for all but two coordinates, leading to a total number of $m = kd + k^2 d(d - 1)/2$ control variates. The maximum degree in each variable is $k=6$, yielding $m=240$ and $m=1056$ control variates in dimensions $d=4$ and $d=8$ respectively. Figure~\ref{fig:synthetic_1} presents the boxplots of the AIS and AISCV estimates. The error reduction obtained thanks to the control variates is huge: the AISCV estimate has a mean squared error smaller than the one of the AIS estimate by a factor at least $10$ and up to $100$ (see Table~\ref{tab:synthetic_1} in the supplement).

	\begin{figure}[h]
		\centering
		\begin{subfigure}[b]{0.245\textwidth}
			\centering
			\includegraphics[width=\textwidth]{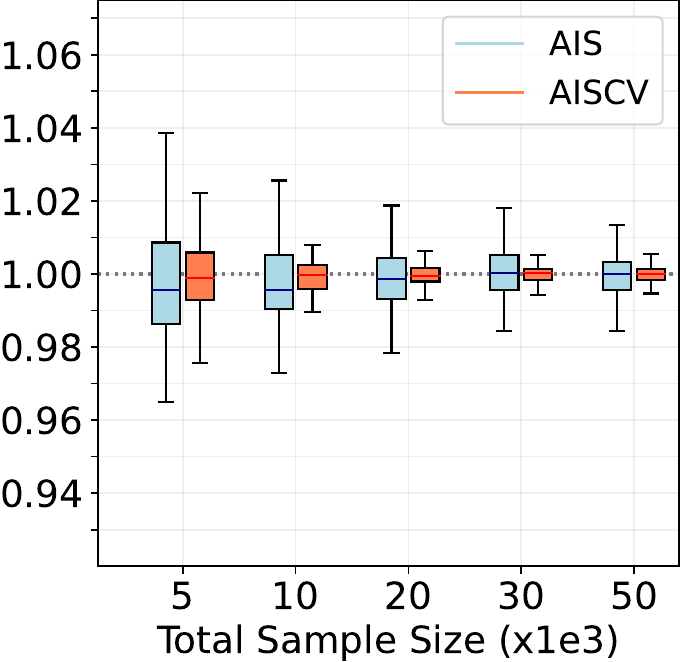}
			\caption{$g_1 (d=4)$}
			\label{fig:sin_d4}
		\end{subfigure}
		\hfill
		\begin{subfigure}[b]{0.245\textwidth}
			\centering
			\includegraphics[width=\textwidth]{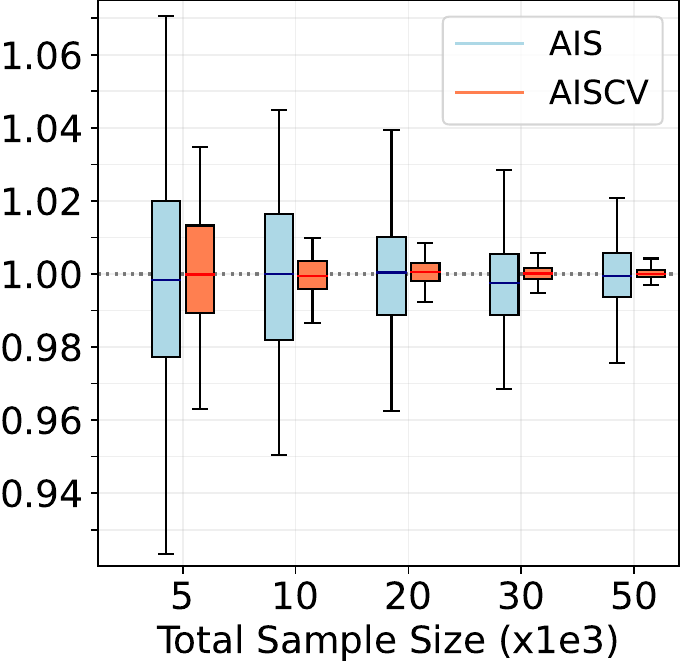}
			\caption{$g_1 (d=8)$}
			\label{fig:sin_d8}
		\end{subfigure}
		\hfill
		\begin{subfigure}[b]{0.245\textwidth}
			\centering
			\includegraphics[width=\textwidth]{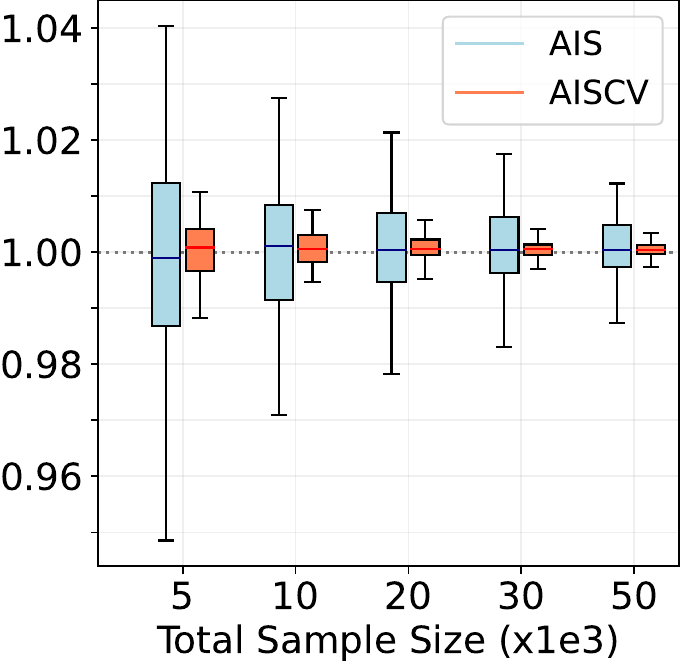}
			\caption{$g_2 (d=4)$}
			\label{fig:lognorm_d4}
		\end{subfigure}
		\hfill
		\begin{subfigure}[b]{0.245\textwidth}
			\centering
			\includegraphics[width=\textwidth]{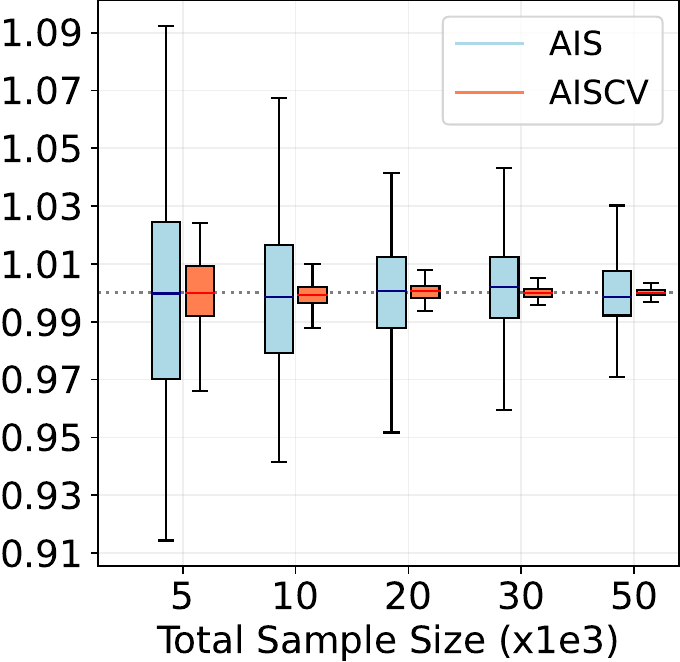}
			\caption{$g_3 (d=8)$}
			\label{fig:exp_d8}
		\end{subfigure}
		\caption{Integration on $[0, 1]^d$: boxplots of estimates $I^{\mathrm{(ais)}}_{n}(g)$ and $I^{\mathrm{(aiscv)}}_{n}(g)$ with integrands $g_1, g_2, g_3$ in dimensions $d \in \{4;8\}$ obtained over $100$ replications. The true integral equals $1$.}
		\label{fig:synthetic_1}        
	\end{figure}
	
\iffalse	
	\begin{table}[h!]
		\centering
		\begin{tabular}{|c|c||c|c|c|c|c|} 
			\hline
			\multicolumn{2}{|c||}{Sample Size $n$} & \multirow{2}{2.7em}{$5,000$} & \multirow{2}{2.7em}{$10,000$} & \multirow{2}{2.7em}{$20,000$} & \multirow{2}{2.7em}{$30,000$} & \multirow{2}{2.7em}{$50,000$} \\ 
			Integrand & Method & & & & & \\
			\hline\hline
			\multirow{2}{3.1em}{\centering $g_1$ \\ $(d=4)$} &  AIS & $2.9e$-$4$ & $1.5e$-$4$ & $7.8e$-$5$ & $5.8e$-$5$ & $3.7e$-$5$ \\
			& AISCV & \bm{$9.7e$}-\bm{$5$} & \bm{$1.9e$}-\bm{$5$} & \bm{$1.0e$}-\bm{$5$} & \bm{$7.5e$}-\bm{$6$} & \bm{$4.3e$}-\bm{$6$} \\ 
			\hline
			\multirow{2}{3.1em}{\centering $g_1$ \\ $(d=8)$} &  AIS & $8.7e$-$4$ & $4.6e$-$4$ & $2.3e$-$4$ & $1.9e$-$4$ & $1.0e$-$4$\\
			& AISCV & \bm{$3.2e$}-\bm{$4$} & \bm{$3.2e$}-\bm{$5$} & \bm{$1.1e$}-\bm{$5$} & \bm{$6.0e$}-\bm{$6$} & \bm{$2.5e$}-\bm{$6$}\\ 
			\hline
			\multirow{2}{3.1em}{\centering $g_2$ \\ $(d=4)$} &  AIS & $3.4e$-$4$ & $1.3e$-$4$ & $7.6e$-$5$ & $5.9e$-$5$ & $3.1e$-$5$\\
			& AISCV &  \bm{$3.1e$}-\bm{$5$} & \bm{$1.0e$}-\bm{$5$} & \bm{$4.9e$}-\bm{$6$} &\bm{$2.6e$}-\bm{$6$} & \bm{$1.5e$}-\bm{$6$}\\ 
			\hline
			\multirow{2}{3.1em}{\centering $g_3$ \\ $(d=8)$} &  AIS & $1.6e$-$3$ & $7.8e$-$4$ & $4.0e$-$4$ & $3.3e$-$4$ & $1.9e$-$4$ \\
			& AISCV & \bm{$1.7e$}-\bm{$4$} & \bm{$2.1e$}-\bm{$5$} & \bm{$7.8e$}-\bm{$6$} & \bm{$4.3e$}-\bm{$6$} & \bm{$1.8e$}-\bm{$6$}\\ 
			\hline
		\end{tabular}
		\caption{Mean Square Errors for $g_1,g_2,g_3$ in dimensions $d \in \{4;8\}$ obtained over $100$ replications.}
	\end{table}
\fi

\textbf{Gaussian mixture $f$ and Stein control variates.} In this setting we assume we only have access to the evaluations of the target density $f$. % and rely on Stein's method for the control variates. 
We consider the classical example introduced in \cite{cappe2008adaptive} where $f$ is a mixture of two Gaussian distributions. The control variates are built using Stein's method (Section~\ref{subsec:cv_practice}) out of polynomials of total degree at most $Q \in \{2;3\}$, leading to a number of control variates $m \in \{14;34\}$ in dimension $d=4$ and $m \in \{44;164\}$ in dimension $d=8$ respectively. We consider two cases: an isotropic and an anisotropic one. %\js{The description of the two cases can perhaps go to the supplement.}
	
%\begin{description}
%\item
{\em Isotropic case.} Let $f_{\Sigma}(x) = 0.5 \Phi_{\Sigma}(x-\mu) + 0.5 \Phi_{\Sigma}(x+\mu)$ where $\mu = (1,\ldots,1)^\top/2\sqrt{d}, \Sigma = I_d/d$ and $\Phi_{\Sigma}$ is the multivariate normal density function with zero mean and covariance matrix $\Sigma$. The Euclidean distance between the two mixture centers is $1$, independently of $d$. The initial density $q_0$ is the multivariate Student $t$ distribution with mean $(1,-1,0,\ldots,0)/ \sqrt{d}$ and variance $(5/d)I_d$. The initial mean value differs from the null vector to prevent the naive algorithm using the initial density from having good results due to the symmetrical set-up. \\ 
{\em Anisotropic case.} In this case, the mixture is unbalanced and each Gaussian is anisotropic. The target density is $f_{V}(x) = 0.75 \Phi_{V}(x-\mu) + 0.25 \Phi_{V}(x+\mu)$ where $\mu = (1,\ldots,1)^\top/2\sqrt{d}$ and $V = \diag(10,1,\ldots,1)/d$. The initial density $q_0$ is the same as for the isotropic case.
%\end{description}

Figure \ref{fig:synthetic_2} presents the evolution of the logarithm of the mean squared error $\| \hat I(g) - I(g)\|_2^2$. Once again, the AISCV estimators are the clear winners with a mean squared error smaller by a factor up to $1000$ for the anisotropic case (see Table~\ref{tab:synthetic_2} in the supplement).
	
	\begin{figure}[h]
		\centering
		\begin{subfigure}[b]{0.245\textwidth}
			\centering
			\includegraphics[width=\textwidth]{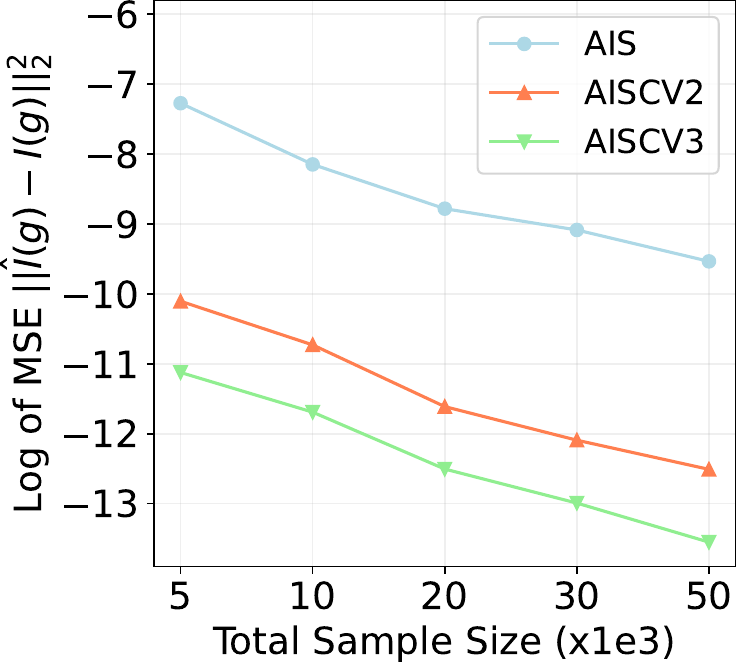}
			\caption{$f_{\Sigma} (d=4)$}
			\label{fig:iso_d4}
		\end{subfigure}
		\hfill
		\begin{subfigure}[b]{0.245\textwidth}
			\centering
			\includegraphics[width=\textwidth]{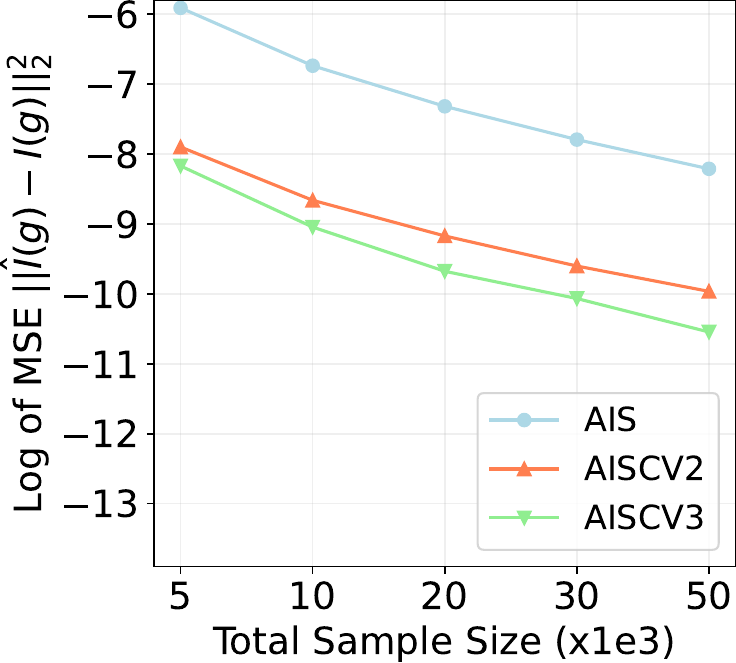}
			\caption{$f_{\Sigma} (d=8)$}
			\label{fig:iso_d8}
		\end{subfigure}
		\hfill
		\begin{subfigure}[b]{0.245\textwidth}
			\centering
			\includegraphics[width=\textwidth]{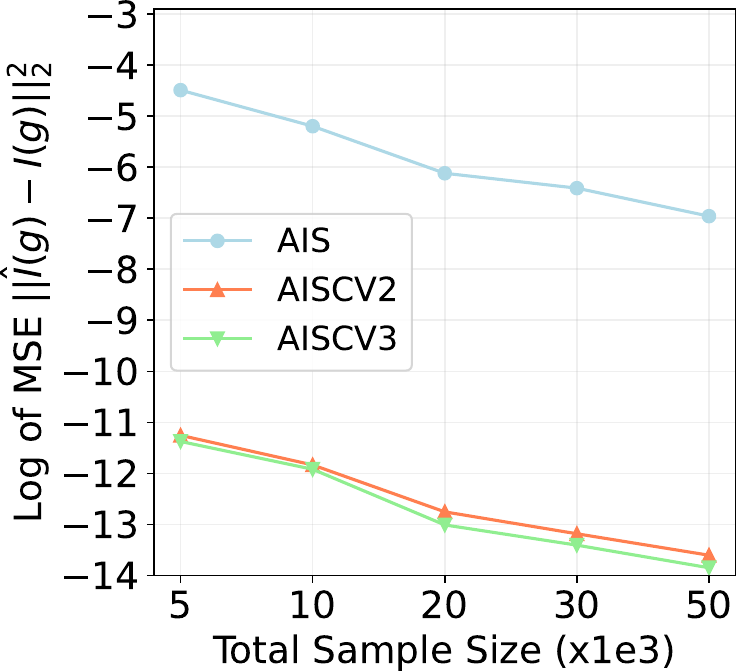}
			\caption{$f_V (d=4)$}
			\label{fig:ani_d4}
		\end{subfigure}
		\hfill
		\begin{subfigure}[b]{0.245\textwidth}
			\centering
			\includegraphics[width=\textwidth]{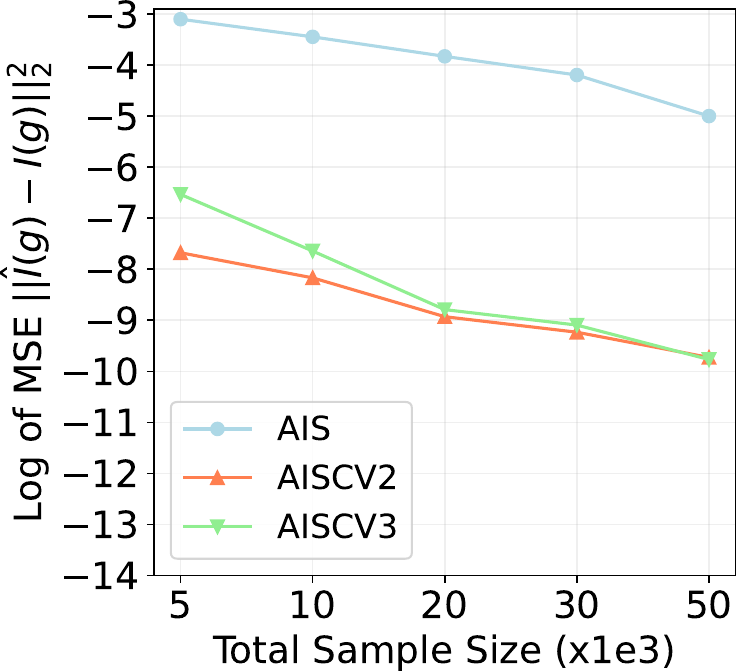}
			\caption{$f_{V} (d=8)$}
			\label{fig:ani_d8}
		\end{subfigure}
		\caption{Gaussian mixture density: Logarithm of $\| \hat I(g) - I(g)\|_2^2$ for $g(x)=x$ with target isotropic $f_{\Sigma}$ and anisotropic $f_V$ in dimensions $d \in \{4;8\}$ obtained over $100$ replications.}
		\label{fig:synthetic_2}        
	\end{figure}

\subsection{Real-world examples} \label{sec:real_world}

We place ourselves in the framework of Bayesian linear regression (Section~\ref{sec:bayes}) with features $X \in \rset^{N \times d}$ and continuous responses $y \in \rset^N$. The posterior distribution $p(\theta|\mathcal{D})$ involves a Gaussian prior $\pi(\theta) \sim \mathcal{N}(\mu_a,\Sigma_a)$ and a likelihood function $\ell(\mathcal{D}|\theta)$ proportional to $(\sigma^2)^{-N/2} \exp(-(y-X \theta)^\top (y-X \theta)/(2 \sigma^2))$. The noise level is fixed and taken sufficiently large at $\sigma = 50$ to account for general priors. %\js{What do you mean? The prior $\pi(\theta) \sim \mathcal{N}(\mu_a, \Sigma_a)$ does not involve $\sigma^2$ at all.} 
The posterior distribution is Gaussian too: $\mathcal{N}(\mu_b,\Sigma_b)$ with $\mu_b = \Sigma_b (\sigma^{-2} X^\top y + \Sigma_a^{-1} \mu_a)$ and $\Sigma_b = (\sigma^{-2} X^\top X + \Sigma_a^{-1})^{-1}$. The integrand is $g(\theta) = \sum_{i=1}^d \theta_i^2$ and the control variates are built with the Stein operator (Section~\ref{subsec:cv_practice}) out of monomials with total degree $Q \in \{1;2\}$, leading to the AISCV1 and AISCV2 estimators respectively. %When $Q=2$, the integrand belongs to the linear span of the control variates so the integration should be exact in the light of Proposition~\ref{prop:exact}.
%\js{This is because the Stein operator associated to the Gaussian density applied to monomials returns a polynomial? Of the same degree? Short explanation to be included.}

\textbf{Datasets and parameters.} Classical datasets from \cite{dua2019uci} are considered : \textit{housing} ($N=506;d=13;m \in \{12;104\}$); \textit{abalone} ($N=\num{4177};d=8;m \in \{7;44\}$); \textit{red wine} ($N=\num{1599};d=11;m \in \{10;77\}$); and \textit{white wine} ($N=\num{4898};d=11;m \in \{10;77\}$). The initial density is the multivariate Student $t$ distribution with $\nu = 10$ degrees of freedom, zero mean and covariance matrix $\Sigma_b$.

	\begin{figure}[h]
		\centering
		\begin{subfigure}[b]{0.245\textwidth}
			\centering
			\includegraphics[width=\textwidth]{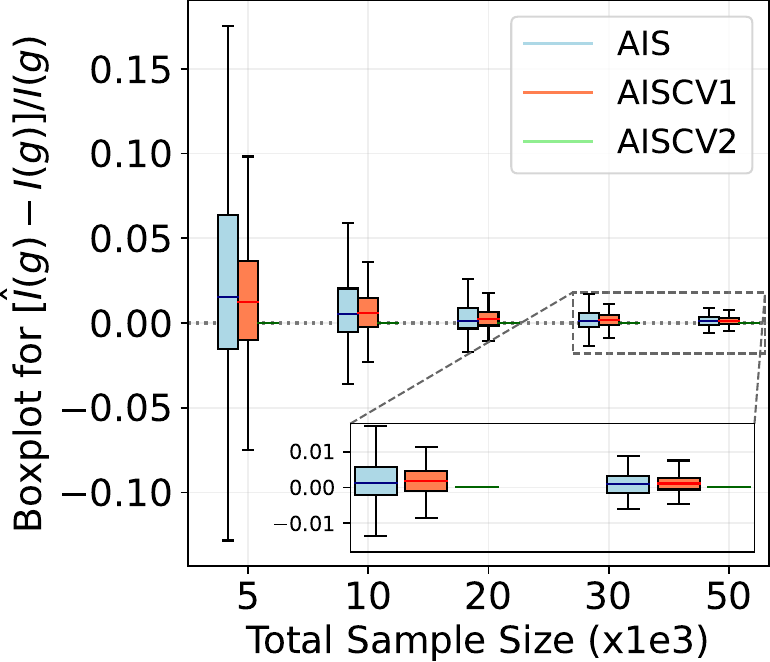}
			\caption{\centering Housing}
			\label{fig:housing}
		\end{subfigure}
		\hfill
		\begin{subfigure}[b]{0.245\textwidth}
			\centering
			\includegraphics[width=\textwidth]{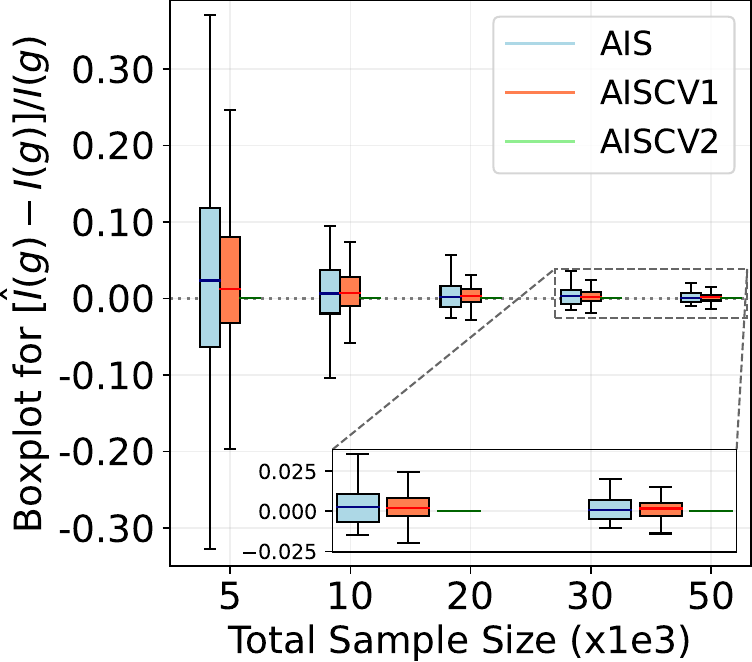}
			\caption{\centering Abalone}
			\label{fig:abalone}
		\end{subfigure}
		\hfill
		\begin{subfigure}[b]{0.245\textwidth}
			\centering
			\includegraphics[width=\textwidth]{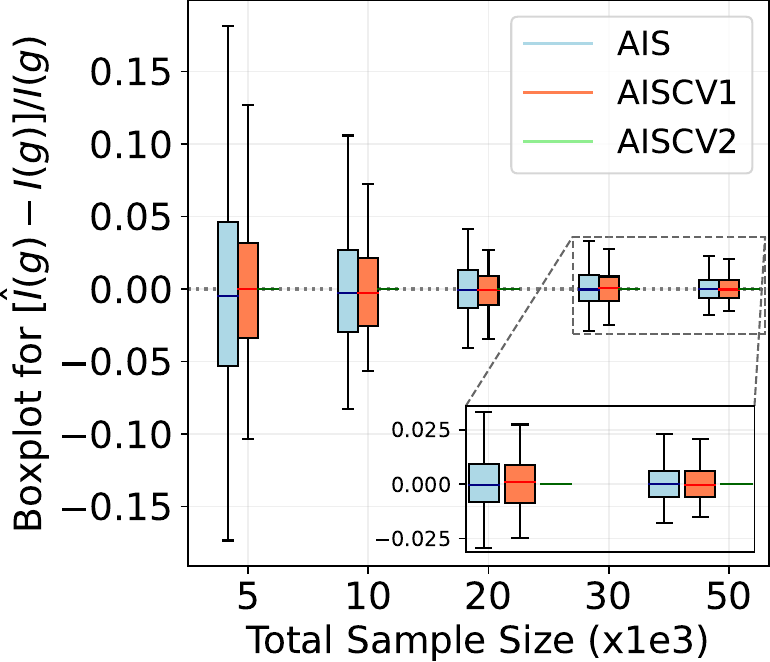}
			\caption{\centering Red Wine}
			\label{fig:mpg}
		\end{subfigure}
		\hfill
		\begin{subfigure}[b]{0.245\textwidth}
			\centering
			\includegraphics[width=\textwidth]{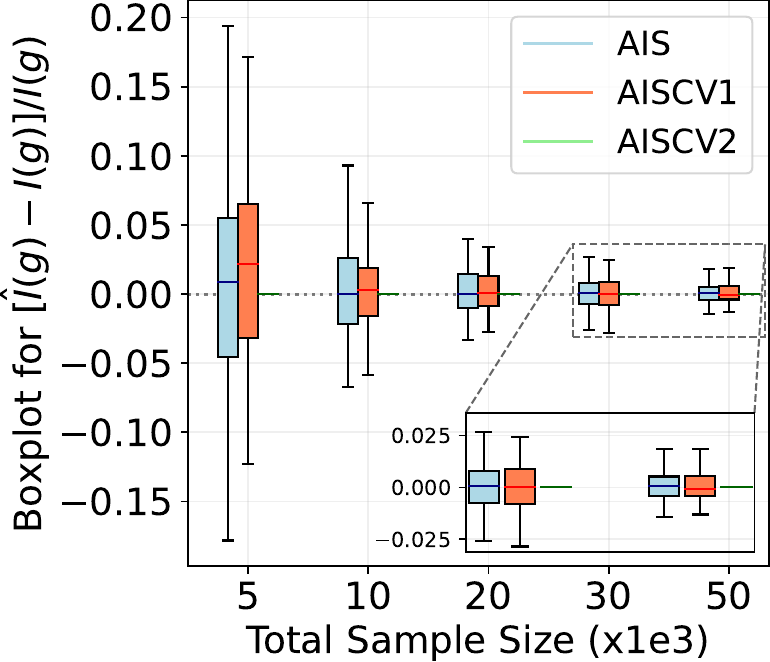}
			\caption{\centering White Wine}
			\label{fig:trya}
		\end{subfigure}
		\caption{Bayesian linear regression: boxplots of $(\hat I(g) - I(g))/I(g)$ for $g(\theta)=\sum_{j=1}^d \theta_j^2$.} %obtained over $100$ replications.}
		\label{fig:BLR_real}        
	\end{figure}
	
\textbf{Results.}
Figure \ref{fig:BLR_real} presents the boxplots of the relative error $(\hat I(g) - I(g))/I(g)$, revealing the benefits of control variates even with polynomials of degree $Q=1$. When $Q=2$, the error of the AISCV2 estimate is virtually zero (see Table~\ref{tab:BLR_real} in the supplement), in line with Proposition~\ref{prop:exact}. The mean squared error of the AISCV1 estimate is smaller than that of the AIS estimate by a factor ranging between $2$ and $10$. %, while for the AISCV2 estimate, the error reduces by a factor of the order $10^{-9}$ 

\section{Discussion} \label{sec:conclu}

While control variates are a well-known tool for Monte Carlo integration, standard methods do not allow the distribution of particles to evolve throughout the algorithm, as is the case for sequential methods. Within the standard adaptive importance sampling framework, we have developed a weighted least-squares procedure to improve numerical integration by incorporating control variates. The underlying adapted weights of this quadrature rule do not depend on the integrand and our non-asymptotic bound highlights the benefits of combining adaptive importance sampling with control variates. Different ways for constructing control variates are proposed. The method is fit for computing integrals with respect to the posterior density in Bayesian analysis, as the target density only needs to be known up to a multiplicative constant.

A limitation of the combined AISCV approach is that it requires the user to make quite some design choices, notably the sampling policy for the AIS part and the control variates for the CV part. These culminate into the factor $\tau$ in Assumption~\ref{as:sub_gauss}, which appears prominently in the error bound in Theorem~\ref{th:concentration_inequality} and which can be interpreted roughly as the standard deviation of $w \varepsilon$, where $w$ is the importance weight -- well behaved when the policy is well-chosen in relation to the target density -- and where $\varepsilon$ is some residual function -- well behaved when the control variates are well-chosen with respect to the integrand. Further, choosing too many control variates may result in an ill-conditioned empirical Gram matrix or in overfitting. The least-squares solution could become unstable, requiring some kind of regularization, such as the LASSO \cite{LelucPortierSegers2021}.

\section*{Acknowledgements}
The authors are grateful to the Area chair and three anonymous Reviewers for their valuable comments and interesting suggestions.
Aigerim Zhuman gratefully acknowledges a research grant from the \emph{National Bank of Belgium}.
%\js{Reference list to be cleaned up: there's a double reference; upper case letters where appropriate (may need to use \{\dots\} in bib file); journal names either in full or abbreviated but not both).}
	
%%%%%%%%%%%%%%%%%%%%%%%%%%%%%%%%%%%% TABLE REAL-WORLD	

%%%%%%%% REFERENCES %%%%%%%%%%%%%%
\bibliographystyle{apalike}
\bibliography{referencelistAISCV}

\newpage

\newpage

\begin{center}
{\large {\bf\textsc{Supplementary material: A Quadrature Rule combining \\ Control Variates and Adaptive Importance Sampling}}}
\end{center}

Technical Lemmas and auxiliary results are provided in Appendix \ref{app:auxiliary_results}. Appendix \ref{app:properties} collects additional theoretical properties of the AISCV estimator while the technical proofs of the Propositions and main theorem are presented in Appendix \ref{app:proofs}. Finally, Appendix \ref{app:num_supp} presents additional numerical values associated to the numerical experiments on synthetic examples and real-world datasets for Bayesian linear regression.

\appendix
%\appendixpage
\startcontents[sections]
\printcontents[sections]{l}{1}{\setcounter{tocdepth}{2}}

\newpage
\section{Auxiliary results} \label{app:auxiliary_results}

\begin{comment}
\begin{theorem} \label{th:freedman} 
(Theorem 17 from \cite{delyon2021safe})
Let $(\Omega,\mathcal{F},(\mathcal{F}_i)_{i \ge 1}, P)$ be a filtered probability space. Let $(Y_i)_{1 \le i \le n}$ be real-valued random variables such that
\begin{equation*}
\E[Y_i \mid \mathcal{F}_{i-1}] = 0, \quad \text{for all } i=1,\dots,n.
\end{equation*}
Then, for all $t \ge 0$ and $v, u \ge 0$,
\begin{equation*}
\prob \left( \left| \sum_{i=1}^{n} Y_i \right| \ge t, \; \max_{i=1,\dots,n} |Y_i| \le u, \; \sum_{i=1}^{n} \E \bigl[ Y_i^2 \mathrel{\big|} \mathcal{F}_{i-1} \bigr]\le v   \right) \le 2 \exp(-\frac{t^2}{2 (v + t u/3 )}).
\end{equation*}
\end{theorem}
\end{comment}

\subsection{Lemmas on (Random) Matrices inequalities}

\begin{definition}
	Let $A$ and $\Psi$ be Hermitian matrices of the same dimension. We say that $A \preceq \Psi$ if and only if $\Psi - A$ is positive semidefinite.
\end{definition}

\begin{definition}[\cite{Tropp2015}, Definition 2.1.2]
	Let $f: I \to \reals$ where $I$ is an interval of the real line. Consider a $d \times d$ Hermitian matrix $A$ whose eigenvalues are contained in $I$. Define a $d \times d$ Hermitian matrix $f(A)$, called the standard matrix function, using an eigenvalue decomposition of $A$, by
	\begin{equation*}
		f(A) = 
		Q 
		\begin{bmatrix}
			f(\lambda_1) &           & \\
			& \ddots    & \\
			&			 & f(\lambda_d)
		\end{bmatrix}
		Q^* 
		\quad\text{where}\quad
		A = 
		Q 
		\begin{bmatrix}
			\lambda_1 &           & \\
			& \ddots    & \\
			&			  & \lambda_d
		\end{bmatrix}
		Q^*. 			
	\end{equation*}
\end{definition}

\begin{remark}
	The matrix exponential $e^A$ and the matrix logarithm $\log(A)$ are the standard matrix functions.
\end{remark}

\begin{lemma}[\cite{Tropp2015}, Example 8.3.4] 
	\label{lm:monotone_map}
	The trace exponential map is monotone:
	\begin{equation*}
		A \preceq \Psi \text{ implies } \tr e^A \leq \tr e^{\Psi}
	\end{equation*}
	for all Hermitian matrices $A$ and $\Psi$.
\end{lemma}

\begin{lemma}[\cite{Tropp2015}, Proposition 3.2.1]
	\label{lm:key}
	For any random Hermitian matrix $Y$, for all $t \in \reals$, we have
	\begin{align*}
		\prob\left( \lambda_{\min}(Y) \leq t \right) \leq \inf_{\theta < 0} e^{-\theta t} \E[\tr(e^{\theta Y})]. 
	\end{align*}
\end{lemma}

\begin{lemma}[\cite{Tropp2015}, Lemma 5.4.1]
	\label{lm:bound_log} 
	Assume that $A$ is a random matrix with $0 \leq \lambda_{\min}(A)$ and, for some constant $L > 0$, $\lambda_{\max}(A) \leq L$. Then, for all $\theta \in \reals$, 
	\begin{equation*}
		\log(\E[e^{\theta A}]) \preceq \eta(\theta) \E[A], \quad \eta(\theta) = L^{-1}(e^{\theta L} - 1).
	\end{equation*}	
\end{lemma}

%\begin{lemma}\citep{tropp2015introduction}[Theorem 3.4.1] \label{lemma:concave}  Fix a Hermitian matrix $H$ with dimension $d$. The function $$A \mapsto tr \exp(H + \log(A))$$ is a concave map on the convex cone of $d \times d$ positive definite matrices.
%\end{lemma}

\begin{lemma}[\cite{Tropp2015}, Corollary 3.4.2]
	\label{lm:concave}   
	Let $\Psi$ be a fixed Hermitian matrix and $A$ a random Hermitian matrix of the same dimension. Then 
	\begin{equation*}
		\E\left[\tr\{\exp(\Psi+A)\} \right] \leq \tr\left[ \exp\{\Psi + \log(\E[e^A])\} \right].
	\end{equation*}
\end{lemma}

\subsection{Inequalities for martingales increments and empirical Gram matrices}
\begin{lemma}[Hoeffding inequality for norm-subGaussian martingale increments]
	\label{lem:HffnsbG}
	Let the $d$-dimensional random vectors $Z_1,\ldots,Z_n$ and the natural filtration $\mathcal F_n  = \sigma (Z_1, \ldots ,Z_n)$, $\mathcal F_0 = \{\Omega, \emptyset\}$,
	be such that, for all $i = 1,\ldots n $,  $ \mathbb E [ Z_i | \mathcal F_{i-1} ]  = 0 $ and
	\begin{equation}
		\label{eq:HffnsbG}
		\forall t \ge 0, \forall i = 1,\ldots,n, \qquad
		\prob(\norm{Z_i}_2 \ge t | \mathcal F_{i-1} ) \le 2 \exp \left( - \tfrac{t^2}{2\sigma^2} \right)
	\end{equation}
	for some $\sigma > 0$. Then for any $\delta > 0$, with probability at least $1-\delta$, we have
	\[
	\norm{\sum_{i=1}^n Z_i}_2 \le K \sigma \sqrt{n \log(2d/\delta)},
	\]
	where $K=3$.
\end{lemma}
\begin{proof}
	The proof follows from adapting the proof of Lemma 6 in \cite{LelucPortierSegers2021} working out their Lemma 5 and Corollary 7 from \cite{jin2019short}.
\end{proof}

\begin{lemma} \label{prop:min_eigenvalue} 
	Define $h_k=h(X_k)$, $Q_k = w_k h_k h_k^\top$, $Y_n = \sum_{k=1}^n Q_k$. Let the constant $L > 0$ be such that $\lambda_{\max}(Q_k) \le L$ with probability $1$. Then, for all $\zeta \in (0,1)$, we have
	\begin{align*}
		\prob\left(\lambda_{\min}(Y_n) \leq (1-\zeta) n \lambda_{\min}(G)\right) \leq m \left[ \frac{e^{-\zeta}}{(1-\zeta)^{(1-\zeta)}}\right]^{n \lambda_{\min}(G)/L}.
	\end{align*}
\end{lemma}

\begin{remark} \label{rk:square_brackets}
	The term in square brackets in Proposition~\ref{prop:min_eigenvalue} is bounded above by $e^{-\zeta^2/2}$ (\cite{LelucPortierSegers2021}, Lemma~2).
\end{remark}

\begin{proof}
	Let $\E_n$ denote the expectation with respect to $\mathcal{F}_{n-1} = \sigma(X_1,\dots,X_{n-1})$ and define $Z_n =\log(\E_n[e^{\nu Q_n}])$. 
	Using Lemma \ref{lm:concave} with the measurable w.r.t. $\mathcal{F}_{n-1}$ matrix $\Psi = \nu Y_{n-1}$, we have
	\begin{align*}
		\E \left[ \tr(e^{\nu Y_n}) \right] 
		= \E\left[ \E_n\left[ \tr(e^{\nu Y_{n-1} + \nu Q_n})\right] \right] 
		\leq \E\left[ \tr(e^{\nu Y_{n-1} + \log(\E_n[e^{\nu Q_n}])})\right] 
		= \E\left[ \tr(e^{\nu Y_{n-1} + Z_n})\right].
	\end{align*}
	Using again Lemma \ref{lm:concave} with the matrix $\Psi = \nu Y_{n-2} + Z_n$, the last term is upper bounded as
	\begin{align*}
		\E\left[ \tr(e^{\nu Y_{n-1} + Z_n})\right] 
		= \E\left[\E_{n-1}\left[ \tr(e^{\nu Y_{n-2} + \nu Q_{n-1} + Z_n})\right] \right] \leq \E\left[ \tr(e^{\nu Y_{n-2} + Z_{n-1} + Z_n})\right]
		%\ldots &\leq tr(e^{\sum_{k=1}^n Z_k}). 
	\end{align*}
	Applying this inequality several times yields
	\begin{align*}
		\E[\tr(e^{\nu Y_n})] \leq \E[\tr( e^{\sum_{k=1}^n Z_{k}})].
	\end{align*}
	Applying Lemma \ref{lm:bound_log} gives $Z_k \preceq \eta(\nu) \E_k[Q_k]$, $\eta(\nu) = L^{-1}(e^{\nu L} - 1)$ for $k=1,\dots,n$. By Lemma \ref{lm:monotone_map}, we get
	\begin{align*}
		\E[\tr(e^{\nu Y_n})] \leq \E[\tr( e^{\sum_{k=1}^n Z_{k}})] \leq \E[\tr( e^{\sum_k \eta(\nu) \E_k[Q_{k}]})] = \tr( e^{n \eta(\nu) G}) .
	\end{align*}
	Now applying Lemma~\ref{lm:key} and taking into account the fact that $\eta(\nu) < 0$ for $\nu < 0$, we have
	\begin{align*}
		\prob\left( \lambda_{\min}(Y_n) \leq t \right) 
		&\leq \inf_{\nu < 0} e^{-\nu t} \E[\tr(e^{\nu Y_n})] \\
		&\leq \inf_{\nu < 0} e^{-\nu t} \tr( e^{n \eta(\nu) G})\\
		&\leq \inf_{\nu < 0} e^{-\nu t} \tr( e^{n \eta(\nu) \lambda_{\min}(G) I_m})\\
		&\leq \inf_{\nu < 0} e^{-\nu t} m e^{n \eta(\nu) \lambda_{\min}(G) }.
	\end{align*}
	We make the change of variables $t = (1-\zeta) n \lambda_{\min}(G) $ and minimize over $\nu < 0$ the following expression
	\begin{equation*}
		-n \nu (1-\zeta) \lambda_{\min}(G) + n \eta(\nu)  \lambda_{\min}(G).
	\end{equation*}
	The infimum is attained at $\nu=L^{-1}\log(1-\zeta)$ with $\eta(\nu)=-\zeta/L$ which gives the inequality of the Lemma.
\end{proof}

\section{Additional properties of AISCV estimator}
\label{app:properties}

\subsection{Orthogonal projections}
Some geometric considerations help to better understand certain properties of the AISCV estimate \eqref{eq:AISCV}. Let $\one_n = (1,\ldots,1)^\top \in \reals^n$ be a vector of ones and write
\begin{align*}
	g^{(n)} &= (g(X_1),\ldots,g(X_n))^\top, &
	H &= (h_j(X_i))_{\substack{i=1,\ldots,n\\j=1,\ldots,m}}, &
	\text{and}\ W &= \diag(w_1,\ldots,w_n).
\end{align*}
In matrix form, the weighted least-squares problem \eqref{eq:AISCV} is
\begin{equation}
\label{eq:AISCV:2}
	(\hat{\alpha}_n, \hat{\beta}_n) 
	\in \argmin_{(a,b) \in \reals \times \reals^m}  \| W^{1/2} (g^{(n)} - a \one_n - H b) \|^2_2.
\end{equation}
For any function $\varphi : \reals^d \to \reals^p$, let the operator $P_{n,w}$ return the weighted average of the sequence $\varphi(X_1), \ldots, \varphi(X_n)$ with the weights $w_1,\dots,w_n$, i.e.,
\begin{equation*}
	P_{n,w}(\varphi) = \frac{\sum_{i=1}^{n} w_i \varphi(X_i)}{\sum_{i=1}^{n} w_i}.
\end{equation*}
The empirically centred integrand and control variates are $g^{(n)}_W = g^{(n)} - \one_n P_{n,w}(g)$ and $H_W = H - \one_n P_{n,w}(h^\top)$. Put $W^{1/2} = \diag(w_1^{1/2}, \ldots, w_n^{1/2})$. The solution to \eqref{eq:AISCV:2} takes the form
\begin{equation}
\begin{cases}
& \hat{\alpha}_n = P_{n,w}(g -  \hat{\beta}_n^\top h) ,\\
& \hat{\beta}_n \in  \argmin_{b \in \reals^m} \| W^{1/2} (g^{(n)}_W - H_W b) \|^2_2,
\end{cases}       
\end{equation}
If the matrix $H_W^\top W H_W$ is invertible, the optimal vector $\hat{\beta}_n$ is unique and is given by
\begin{equation} \label{eq:wls_alpha_beta}
	\hat{\beta}_n = (H_W^\top W H_W)^{-1} H_W^\top W g^{(n)}_W.
\end{equation}

\subsection{Matrix representation}
Let us rewrite \eqref{eq:AISCV:2} in terms of two nested minimization problems:
\begin{equation}
\label{eq:aiscv:nested}
\hat{\alpha}_n \in \argmin_{a \in \reals} \left[ \min_{b \in \reals^m} \left\| W^{1/2} \left(g^{(n)} -  a \one_n - H b\right) \right\|^2_2 \right].
\end{equation}
Let $\Pi \in \reals^{n \times n}$ be the orthogonal projection matrix onto the column space of $H$, when $\reals^n$ is endowed with the scalar product $\langle x, y \rangle_W = x^\top W y$ for $x, y \in \reals^n$. For $v \in \reals^n$, we have
\[
\Pi v = H \hat{\beta}_n(v) 
\quad \text{where} \quad 
\hat{\beta}_n(v) \in \argmin_{b \in \reals^m} \left\| W^{1/2} (v - H b) \right\|_2^2.
%	\sum_{i=1}^n w_i \left(v_i - b^\top h(X_i)\right)^2.
\]
If $H$ has rank $m$, then the solution to the above minimization problem is unique and $\Pi = H (H^\top W H)^{-1} H^\top W$; otherwise, the matrix $\Pi$ is still uniquely defined, even though there are then multiple solutions $\hat{\beta}_n(v)$. Given $a \in \reals$, the minimum in \eqref{eq:aiscv:nested} over $b \in \reals^m$ is attained as soon as $H b = \Pi (g^{(n)} - a \one_{n})$. Therefore
\begin{equation} \label{eq:hat_alpha_I_Pi}
\hat{\alpha}_n 
\in \argmin_{a \in \reals} \left\| W^{1/2} (I_n -\Pi) \left(g^{(n)} - a \one_{n}\right) \right\|^2_2,
\end{equation}
where $I_n$ is the $n \times n$ identity matrix. Recall the vector $e_n$ in \eqref{eq:eni}. In our present notation, we have
\[
	e_n = (I_n - \Pi) \one_n.
\]

\begin{proposition}[Matrix representation]
\label{prop:another_representation}
	The minimizer $\hat{\alpha}_n$ in \eqref{eq:hat_alpha_I_Pi} is unique if and only if $e_n \ne 0$, in which case the normalized AISCV estimate is
	\begin{equation}
	\label{eq:aiscv:mat}
	I_n^{\mathrm{(aiscv)}}(g) =
	\hat{\alpha}_n 
	= \frac{\one_n^\top (I_n - \Pi)^\top W (I_n-\Pi) g^{(n)}}{\one_n^\top (I_n - \Pi)^\top W (I_n-\Pi) \one_n}
	= \frac{\one_n^\top (I_n - \Pi)^\top W g^{(n)}}{\one_n^\top (I_n - \Pi)^\top W \one_n}.
	\end{equation}
\end{proposition}

\begin{proof}
The objective function on the right-hand side of \eqref{eq:aiscv:mat} is
\[
a^2 \one_n^\top (I_n - \Pi)^\top W (I_n - \Pi) \one_n - 2 a \one_n^\top (I_n - \Pi)^\top W (I_n - \Pi) g^{(n)} + \text{constant},
\]
where the unspecified constant does not depend on $a$. The coefficient of $a^2$ is equal to $e_n^\top W e_n$, which is positive if and only if $e_n \ne 0$. The latter is thus a necessary and sufficient for the minimizer $\hat{\alpha}_n$ to exist and be unique. In that case, the objective function is a convex quadratic function in $a$, whose minimizer is easily seen to be equal to the stated expression.
\end{proof}

\section{Proofs of the main results} \label{app:proofs}

%\subsection{Proof of Proposition~\ref{prop:another_representation}}

\subsection{Proof of Proposition~\ref{prop:AISCV:quad}}

\begin{proof}
We start from Proposition~\ref{prop:another_representation}. Recall that $e_n = (I_n - \Pi) \one_n$. Since $\Pi^\top W = W \Pi$ and $\Pi^2 = \Pi$, we find $(I_n - \Pi)^\top W (I_n - \Pi) = (I_n - \Pi)^\top W$. We obtain
	\[
	\one_n^\top (I_n - \Pi)^\top W (I_n - \Pi) g^{(n)}
	= \one_n^\top (I_n - \Pi)^\top W g^{(n)}
	= e_n^\top W g^{(n)} = \sum_{i=1}^n w_i e_{n,i} g(X_i),
	\]
	and similarly $\one_n^\top (I_n - \Pi)^\top W (I_n - \Pi) g^{(n)} = \sum_{i=1}^n w_i e_{n,i}$.
\end{proof}

\subsection{Proof of Proposition~\ref{prop:exact}}
\begin{proof}
	If $g = \alpha + \beta^\top h$ for some $\alpha \in \reals$ and $\beta \in \reals^m$, then the minimum in \eqref{eq:AISCV} is clearly attained for $\hat{\alpha}_n = \alpha$ and $\hat{\beta}_n = \beta$.
\end{proof}

\subsection{Proof of Proposition~\ref{prop:invariance}}
\begin{proof}
	In \eqref{eq:AISCV}, if $b$ ranges over $\reals^m$, then $A^\top b$ ranges over $\reals^m$ too, since $A$ is invertible. It follows that the solutions $\hat{\alpha}_n$ in \eqref{eq:AISCV} do not change if we replace $h$ by $Ah$, since $b^\top A h = (A^\top b)^\top h$.
\end{proof}

\subsection{Proof of Theorem \ref{th:concentration_inequality}}

\begin{proof}
\item
	\paragraph{Step 1: Working out the probability of several bounds.} In Step 1, we gather several elementary bounds that will be useful to establish more advanced bounds in Step 2.

	\noindent \textbf{Bound 1.}
	To control $\left|  \sum_{i=1}^{n} w_i \varepsilon( X_i) \right|$, we apply Lemma \ref{lem:HffnsbG} with $Z_i $ equal to $ w_i \varepsilon(X_i)$. We have $\mathbb E [ w_i \varepsilon(X_i)|\mathcal F_{i-1}] = 0$ and by Assumption \ref{as:sub_gauss},
	\begin{align*}
		\mathbb P [ | w _ i \varepsilon (X_i)  | >t  | \mathcal F_{i-1} ] \leq 2  \exp( -  t  ^2 / ( 2 \tau^2)  ) 
	\end{align*}
	holds, and the sub-Gaussian variance factor is simply $ \tau^2$.
	Therefore, with probability at least $1 - \delta / 5$, we have
	\begin{align*}
		\left|  \sum_{i=1}^{n} w_i \varepsilon( X_i)   \right| \leq   K   \tau  \sqrt{n \log(10/\delta)}.
	\end{align*}

	\noindent \textbf{Bound 2.} For the term $\left\|     \sum_{i=1}^{n} w_i \hbar (X_i) \right\|_2$, we apply Lemma \ref{lem:HffnsbG} with $ Z_ i $ equal to $ w_i \hbar (X_i) $.  By Assumptions \ref{as:bounded_cv} and \ref{as:bound_of_density}, we have $\|  w_i \hbar (X_i) \| _2 \leq c \| \hbar(X_i) \|_2\leq c \sqrt B$, which implies that $ w_i \hbar (X_i)$ is sub-Gaussian (conditionally on $\mathcal F_{i-1}$) with variance factor $c ^2 B  $ \citep[Lemma 2.2]{boucheron2013concentration}. Hence \eqref{eq:HffnsbG} is satisfied with $\sigma^2 =c ^2 B $. Thus, with probability at least $1-\delta / 5$, the inequality
	\begin{align*}
		\left\|     \sum_{i=1}^{n} w_i \hbar (X_i) \right\|_2 \leq  K  c   \sqrt{n B \log(10m/\delta)}
	\end{align*}
	holds.
	
	\noindent \textbf{Bound 3.} Now we treat the term $\left\| \sum_{i=1}^{n} w_i  \hbar (X_i) \varepsilon (X_i)  \right\|_2$ applying again Lemma \ref{lem:HffnsbG} but this time with $Z_i$ equal to $ w_i \hbar (X_i) \varepsilon (X_i)  $. We have that $ \|  w_i \hbar (X_i) \varepsilon_ i \| _2 \leq  \sqrt B |w_i \varepsilon_ i| $.  By Assumption \ref{as:sub_gauss}, we have, for all $t > 0$,
	\begin{align*}
		\prob[  \norm{w_i \hbar(X_i) \varepsilon(X_i)}_2 > t | \mathcal F_i  ]
		&\le
		\prob[ \sqrt{B} \abs{w_i \varepsilon(X_i)} > t   | \mathcal F_i  ] \\
		&\le 
		2 \exp\left( - \frac{t^2}{2 B \tau^2} \right),
	\end{align*}
	and \eqref{eq:HffnsbG} holds with $\sigma^2 =  B \tau^2$. Lemma~\ref{lem:HffnsbG} then implies that, with probability at least $1-\delta/5$,
	\begin{equation*}
		\norm{ \sum_{i=1}^n  w_i  \hbar(X_i) \varepsilon(X_i) }_2
		\le K  \sqrt{n B \tau^2 \log(10 m/\delta)}. 
	\end{equation*}

	\noindent \textbf{Bound 4.}  By Lemma~\ref{prop:min_eigenvalue} and Remark~\ref{rk:square_brackets}, we have, with probability at least $1-\delta/5$,
	\begin{equation*}
		\lambda_{\min}\left(\sum_{i=1}^{n} w_i \hbar(X_i) \hbar^\top(X_i) \right) > (1-\zeta) n \lambda_{\min}(G) = (1-\zeta) n,
	\end{equation*}
	where, by Assumption~\ref{as:bounded_cv}, $G=\int f \hbar \hbar^\top \ \diff \lambda = I$, $\zeta$ satisfies the equation
	\begin{equation*}
		m \exp(\frac{-\zeta^2 n }{2 L}) = \delta/5.
	\end{equation*}
	with $L=c B$ according to Assumptions~\ref{as:bounded_cv} and \ref{as:bound_of_density}.
	Solving the last equation, we obtain
	\begin{equation*}
		\zeta = \sqrt{\frac{2 L \log(5 m/\delta)}{n}}. 
	\end{equation*}
	We choose $\zeta \le 1/2$ which gives the condition $n \geq 8 c B \log(5 m/\delta)$ and, with probability at least $1-\delta/5$,
	\begin{equation} \label{eq:lambda_min}
		\lambda_{\min}\left(\sum_{i=1}^{n} w_i \hbar(X_i) \hbar^\top(X_i) \right)> \left(1-\zeta\right) n \ge n/2.
	\end{equation}

	\noindent \textbf{Bound 5.}  Now we consider the term $\sum_{i=1}^{n} w_i$. 
	Since $- 1 \leq w_i - 1\leq c $, $|w_i-1| $ is bounded by $ c$, and $w_i - 1$ is sub-Gaussian with variance factor $ c^2$. This makes the inequality required in Lemma \ref{lem:HffnsbG} valid and henceforth
	\begin{align*}
		\left|  \sum_{i=1}^{n} (w_i  - 1)   \right| \leq   K  c     \sqrt{n \log(10/\delta)} 
	\end{align*}
	or
	\begin{align*}
		- K  c     \sqrt{n \log(10/\delta)} + n \leq \sum_{i=1}^{n} w_i   \leq   K  c     \sqrt{n \log(10/\delta)} + n. 
	\end{align*}
	
	We want to have $K  c     \sqrt{n \log(10/\delta)} \le n/2$. It holds if $\sqrt n \geq 2 K c \sqrt {\log(10/\delta)} $. Then we get that $ n/2 = n - n /2 \leq n -  K  c     \sqrt{n \log(10/\delta)}\leq  \sum_{i=1}^{n} w_i   $. Therefore, with probability at least $1-\delta / 5$, it holds that
	\begin{align*}
		\sum_{i=1}^{n} w_i \geq n / 2.
	\end{align*}	
	
\item
	\paragraph{Step 2: Extending the previous elementary bounds on appropriate quantities.} The work in this step consists in showing that under the five previous bounds, and therefore with probability at least $1-\delta$, we have that
	\begin{align}
		\label{extend1} &  \lambda_{\min}  \left(  \sum_{i=1}^{n} w_i \hbar_W (X_i) \hbar_W (X_i) ^\top \right) \geq  n / 4, \\
	\label{extend2}   &\left\| \sum_{i=1}^{n} w_i  \hbar_W (X_i) \varepsilon_W(X_i)  \right\|_2 \leq 2 K   \tau  \sqrt{n B \log(10m/\delta)} .
	\end{align}

	We start by proving  \eqref{extend1}. Recognizing a covariance, we get  
	\begin{align*}
		P_{n,w} \{   \hbar_W  \hbar_W ^\top \}  =  P_{n,w} ( \hbar \hbar^\top ) - P_{n,w} (\hbar ) P_{n,w} (\hbar ) ^\top ,
	\end{align*}
	and then, using Cauchy-Schwarz inequality, we have 
	\begin{align*}
		\lambda_{\min} (P_{n,w} \{   \hbar_W  \hbar_W ^\top \}  )  \geq    \lambda _{\min} (P_{n,w} ( \hbar \hbar^\top ) )  -  \| P_{n,w} (\hbar ) \|_2^2 
	\end{align*}
	or, equivalently,
	\begin{align*}
		\lambda_{\min}  \left(  \sum_{i=1}^{n} w_i \hbar_W (X_i) \hbar_W (X_i) ^\top \right)  \geq   \lambda_{\min}  \left(  \sum_{i=1}^{n} w_i \hbar (X_i) \hbar (X_i) ^\top \right) -  \Big\|  \sum_{i=1}^{n} w_i \hbar (X_i)   \Big\|_2^2 \Big/ \sum_{i=1}^{n} w_i ,
	\end{align*}
	From Bound 2 and Bound 5, 
	$$\Big\|  \sum_{i=1}^{n} w_i \hbar (X_i)   \Big\|_2^2 \Big/ \sum_{i=1}^{n} w_i  \leq  \frac{K^2  c^2   B   n  \log(10m/\delta)}{n / 2}   =  2 K^2  c^2   B     \log(10m/\delta)         $$
	Using Bound 4 and the previous inequality, it follows that
	$$
	\lambda_{\min}  \left(  \sum_{i=1}^{n} w_i \hbar_W (X_i) \hbar_W (X_i) ^\top \right)  \geq n /2  -   2 K^2  c^2   B     \log(10m/\delta).   
	$$
	If $n \ge 8 K^2 c^2   B     \log(10m/\delta) $, 
	$$
	\lambda_{\min}  \left(  \sum_{i=1}^{n} w_i \hbar_W (X_i) \hbar_W (X_i) ^\top \right) \geq n/4.
	$$
	We have just obtained  \eqref{extend1}.
	
	Let us now establish \eqref{extend2}. Recognizing a covariance, we find  
	\begin{align*}
		P_{n,w} \{   \hbar_W  \varepsilon_W \}  =  P_{n,w} ( \hbar \varepsilon) - P_{n,w} (\hbar ) P_{n,w} (  \varepsilon) ,
	\end{align*}
	and it follows that
	\begin{align*}
		\left\|  P_{n,w} \{   \hbar_W  \varepsilon_W \}  \right\|_2 \leq \|  P_{n,w} ( \hbar \varepsilon) \|_2 + \| P_{n,w} (\hbar ) \|_2 | P_{n,w} (  \varepsilon) |,
	\end{align*}
	or, equivalently,
	\begin{align*}
		\left\| \sum_{i=1}^{n} w_i  \hbar_W (X_i) \varepsilon_W(X_i)  \right\|_2 
		\leq \Big\| \sum_{i=1}^{n} w_i  \hbar (X_i) \varepsilon (X_i)   \Big\|_2 +  \| P_{n,w} (\hbar ) \|_2    \Big|  \sum_{i=1}^{n} w_i  \varepsilon (X_i)  \Big|  .
	\end{align*}
	Now using Bound 2 and 5, we find
	\begin{align}\label{eq:useful}
		\| P_{n,w} (\hbar ) \|_2   \leq 2 K c  \sqrt{ \frac{B \log(10m/\delta)}{n}} ,
	\end{align}
	which combined with Bound 1 leads to

	\begin{align*}
		\| P_{n,w} (\hbar ) \|_2    \Big|  \sum_{i=1}^{n} w_i  \varepsilon (X_i)  \Big|
		& \leq 2 K^2 c  \tau     \sqrt{ B \log(10m/\delta)  \log(10/\delta)}\\
		& \leq 2 K^2  c  \tau   \sqrt B    \log(10m/\delta) . 
	\end{align*}
	The previous inequality and Bound 3 gives
	\begin{align*}
		\left\| \sum_{i=1}^{n} w_i  \hbar_W (X_i) \varepsilon_W(X_i)  \right\|_2 
&\leq 		 K    \tau  \sqrt{n B\log(10m/\delta)}  +  2 K^2 c  \tau   \sqrt B    \log(10m / \delta)  \\
		& =  K    \tau  \sqrt{n B \log(10m/\delta)} \left( 1 + 2 K c \sqrt {\frac{\log(10m/\delta) }{n} }   \right) \\
		& \leq 2 K   \tau  \sqrt{n B \log(10m/\delta)} 
	\end{align*}
	if $n \geq 4 K^2 c^2 \log(10m/\delta) $. 
	
	The condition $ n\geq 8 K^2  c^2   B     \log(10m/\delta) $ (used in establishing \eqref{extend1}) implies $n \geq 4 K^2 c^2  \log(10m/\delta) $ (used in proving \eqref{extend2}), $n \geq 8 c B \log(5 m/\delta)$ (used in Bound 4) and $n\geq 4K^2 c^2 \log(10/\delta) $ (used in Bound 5) since $ m\geq 1$, $B\geq m$ and $c\geq 1$. Therefore, the constant $C_1$ from the statement of the theorem equals $8 K^2$.

\item
	\paragraph{Step 3. End of the proof.}
	The quantity to be bounded can be written as a sum of two terms as follows
	\begin{align*}
		I^{\mathrm{(aiscv)}}_{n}(g, \hat{\beta}_n)
		- \int_{\reals^d} g(x) f(x) \, \diff x = P_{n,w} \{ \varepsilon \}    +  P_{n,w} \{ h   \}^\top    (\beta^*-\hat{\beta}_n) .
	\end{align*}
	%Conveniently, the first term shall be bounded with probability greater than $1-\delta/4$ and the second term with probability greater than $1-3\delta/4$.
	Using Bounds 1 and 5,  the first term in the right-hand side satisfies

	\begin{align*}
		| P_{n,w} \{ \varepsilon \}   | \leq      2 K     \tau  \sqrt{ \frac{\log(10/\delta)}{n} }  .
	\end{align*}

	This corresponds to the first term in the bound of the theorem with the constant $C_2$ equals $2K$.
	Hence, it remains to show that
	\begin{align*}
		| P_{n,w} \{ h   \}^\top    (\beta^*-\hat{\beta}_n) | \leq C_3 c B \tau \log(10m/\delta) / n.
	\end{align*}
	Introducing $G^{- 1/2} G^{1/2}$, we obtain
	\begin{align*}
		P_{n,w} \{ h   \}^\top (\beta^*-\hat{\beta}_n) 
		=    P_{n,w} \{ \hbar   \}^\top   G^{1/2} (\beta^*-\hat{\beta}_n)  .
	\end{align*}
	Then, using the identity 
	\begin{align*}
		(\hat{\beta}_n -\beta^* ) 
		&=    (H_W^\top W H_W)^{-1} H_W^\top W \varepsilon^{(n)}_W  
	\end{align*}
	and Cauchy-Schwarz inequality yields
	\begin{align*}
		\left | P_{n,w} \{ h   \}^\top  (\beta^*-\hat{\beta}_n) \right|
		&\leq 
		\left\|    P_{n,w} \{ \hbar  \} \right\|_2 \|    G^{1/2} (\beta^*-\hat{\beta}_n)  \|_2\\
		&\leq \left\|    P_{n,w} \{ \hbar   \} \right\|_2 \left\|  G^{1/2}  (H_W^\top W H_W)^{-1} H_W^\top W \varepsilon^{(n)}_W \right\|_2 \\
		&\leq \left\|     P_{n,w} \{ \hbar   \} \right\|_2 \left\|  G^{1/2}  (H_W^\top W H_W)^{-1}  G^{1/2}  \right\|_2 \left\|G^{-1/2}   H_W^\top W \varepsilon^{(n)}_W \right\|_2\\
		& = \left\|     P_{n,w} \{ \hbar   \} \right\|_2 \left\|  G^{1/2}  (H_W^\top W H_W)^{-1}  G^{1/2}  \right\|_2 \left\|G^{-1/2}   H_W^\top W \varepsilon^{(n)}_W \right\|_2.
	\end{align*}
	By \eqref{extend1}, we have
	\begin{align*}
		\left\|  G^{1/2}  (H_W^\top W H_W)^{-1}  G^{1/2}  \right\|_2 &=  \left\| 
		\left( \sum_{i=1}^{n} w_i \hbar_W (X_i) \hbar_W (X_i) ^\top
		\right) ^{-1}\right\|_2 \\
		& =\Big[\lambda _{\min} \left( \sum_{i=1}^{n} w_i \hbar_W (X_i) \hbar_W (X_i) ^\top \right)\Big]^{-1}
		\leq 4/n.
	\end{align*}
	From  \eqref{extend2} and \eqref{eq:useful}, it follows that

	\begin{align*}
		\left | P_{n,w} \{ h   \}^\top  (\beta^*-\hat{\beta}_n) \right|
		&\leq  2 K   \sqrt{ \frac{B \log(10m/\delta)}{n}}   \frac{8 K c \tau  \sqrt{n B \log(10m/\delta)} }{n} \\
		& = 16 K^2 c B \tau \frac{ \log(10m/\delta)}{n}.
	\end{align*}
	
	Therefore, the constant $C_3$ from the statement of the theorem equals $16 K^2$.
\end{proof}

\section{Additional numerical results} \label{app:num_supp}

\textbf{Parameters.} In all simulations, the sampling policy is taken within the family of multivariate Student $t$ distributions of degree $\nu$ denoted by $\{q_{\mu,\Sigma_0}\,:\, \mu \in \reals^d \}$ with $\Sigma_0 =\sigma_0 I_d(\nu-2)  /\nu$ and $\nu >2, \sigma_0 >0$.
Similarly to \cite{portier2018asymptotic}, the mean $\mu_t$ is updated at each stage $t=1,\ldots, T$ by the generalized method of moments (GMM), leading to $$\mu_t	= \frac{\sum_{s=1}^t \sum_{i=1}^{n_s} w_{s,i }X_{s,i}}{\sum_{s=1}^t \sum_{i=1}^{n_s} w_{s,i}}.$$
 The allocation policy is fixed to $n_t = \num{1000}$ and the number of stages is $T \in \{5;10;20;30;50\}$. The different Monte Carlo estimates are compared by their mean squared error (MSE) obtained over $100$ independent replications. 	In other words, for each method that returns $\hat I(g)$, the mean square error is computed as the average of $\|\hat I(g) - I(g) \|_2^2$ computed over $100$ replicates of $\hat I(g)$. When the integrand is real-valued, this quantity is scaled as $([\hat I(g) - I(g)]/I(g))^2$. 
 
 The experiments were performed on a laptop Intel Core i7-10510U CPU \@ 1.80GHz $\times$ 8.
	
\subsection{Synthetic examples: integration on $[0,1]^d$}

We seek to integrate functions $g$ with respect to the uniform density $f(x) = 1$ for $x \in [0, 1]^d$ in dimensions $d \in \{4;8\}$. We rely on Legendre polynomials for the control variates. Consider the integrands $g_1(x) = 1 + \sin(\pi(2 d^{-1} \sum_{i=1}^d x_{i} - 1))$, $g_2(x) = \prod_{i=1}^d (2/\pi)^{1/2} x_i^{-1} \mathrm{e}^{-\log(x_i)^2/2}$ and  $g_3(x) = \prod_{i=1}^d \log(2) 2^{1-x_i}$, all of which integrate to $1$ on $[0, 1]^d$. 
%The functions $g_2$ and $g_3$ are built using tensor products of log-normal and exponential density functions, respectively. 
None of the integrands is a linear combination of the control variates. The policy parameters are $\mu_0 = (0.5,\ldots, 0.5)\in \reals^d$, $\nu = 8$, and $\sigma_0 = 0.1$. The control variates are built out of tensor products of Legendre polynomials where the degree $\ell_j$ equals $0$ for all but two coordinates, leading to a total number of $m = kd + k^2 d(d - 1)/2$ control variates. The maximum degree in each variable is $k=6$, yielding $m=240$ and $m=1056$ control variates in dimensions $d=4$ and $d=8$ respectively. Figure~\ref{fig:synthetic_1} presents the boxplots of the AIS and AISCV estimates.
 
  Figure \ref{fig:app_synthetic_1} presents the boxplots of the different estimates and Table \ref{tab:synthetic_1} gathers the numerical values of the mean squared errors. As a natural competitor to our AISCV estimator, we also implemented the weighted version of standard AIS called \textit{w-AIS} introduced in \cite{portier2018asymptotic}. Interestingly, such a method presents similar or even worse performance than the standard AIS estimate for dimension $d=4$ but better results for dimension $d=8$. This good behavior is illustrated in Figure \ref{fig:sin_d8_wais}
 and Figure \ref{fig:exp_d8_wais}. Accordingly, the values of the MSE for w-AIS are smaller than the one of AIS in dimension $d=8$ but still greater than the ones of AISCV.
 
	\begin{figure}[h]
		\centering
		\begin{subfigure}[b]{0.245\textwidth}
			\centering
			\includegraphics[width=\textwidth]{sin_d4.pdf}
			\caption{$g_1 (d=4)$}
			\label{fig:sin_d4}
		\end{subfigure}
		\hfill
		\begin{subfigure}[b]{0.245\textwidth}
			\centering
			\includegraphics[width=\textwidth]{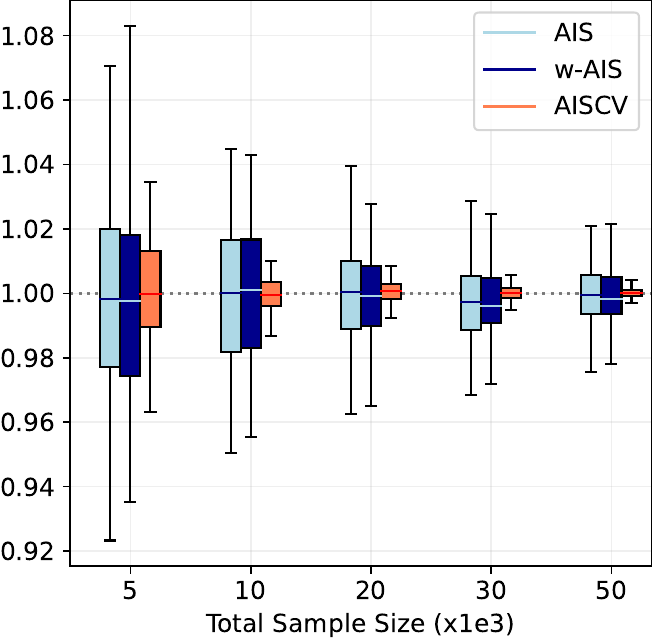}
			\caption{$g_1 (d=8)$}
			\label{fig:sin_d8_wais}
		\end{subfigure}
		\hfill
		\begin{subfigure}[b]{0.245\textwidth}
			\centering
			\includegraphics[width=\textwidth]{lognorm_d4.pdf}
			\caption{$g_2 (d=4)$}
			\label{fig:lognorm_d4}
		\end{subfigure}
		\hfill
		\begin{subfigure}[b]{0.245\textwidth}
			\centering
			\includegraphics[width=\textwidth]{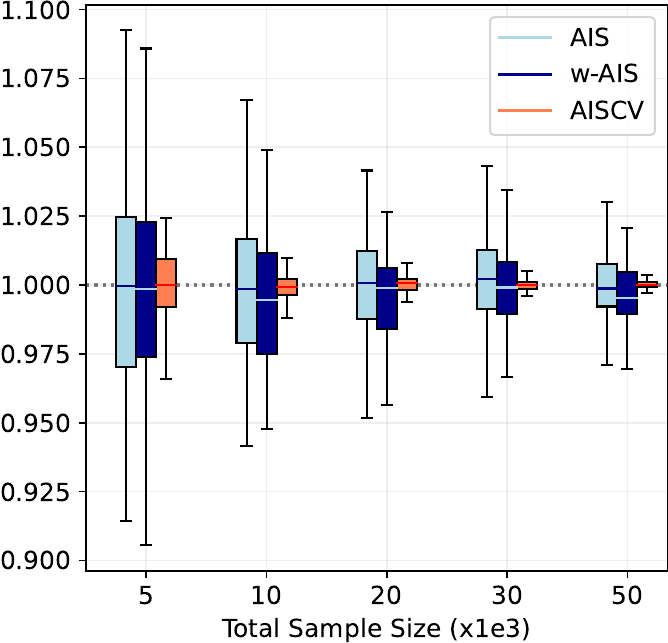}
			\caption{$g_3 (d=8)$}
			\label{fig:exp_d8_wais}
		\end{subfigure}
		\caption{Integration on $[0, 1]^d$: boxplots of estimates $I^{\mathrm{(ais)}}_{n}(g)$ and $I^{\mathrm{(aiscv)}}_{n}(g)$ with integrands $g_1, g_2, g_3$ in dimensions $d \in \{4;8\}$ obtained over $100$ replications. The true integral equals $1$.}
		\label{fig:app_synthetic_1}        
	\end{figure}
	
	\begin{table}[h!]
		\centering
		\begin{tabular}{ccccccc} 
			\hline
			\multicolumn{2}{c}{Sample Size $n$} & \multirow{2}{2.7em}{$5,000$} & \multirow{2}{2.7em}{$10,000$} & \multirow{2}{2.7em}{$20,000$} & \multirow{2}{2.7em}{$30,000$} & \multirow{2}{2.7em}{$50,000$} \\ 
			Integrand & Method & & & & & \\
			\hline
			\multirow{3}{3.1em}{\centering $g_1$ \\ $(d=4)$} &  AIS & $2.9e$-$4$ & $1.5e$-$4$ & $7.8e$-$5$ & $5.8e$-$5$ & $3.7e$-$5$ \\
			& wAIS & $3.0e$-$4$ & $1.6e$-$4$ & $8.3e$-$5$ & $6.5e$-$5$ & $4.1e$-$5$ \\
			& AISCV & \bm{$9.7e$}-\bm{$5$} & \bm{$1.9e$}-\bm{$5$} & \bm{$1.0e$}-\bm{$5$} & \bm{$7.5e$}-\bm{$6$} & \bm{$4.3e$}-\bm{$6$} \\ 
			\hline
			\multirow{3}{3.1em}{\centering $g_1$ \\ $(d=8)$} &  AIS & $8.7e$-$4$ & $4.6e$-$4$ & $2.3e$-$4$ & $1.9e$-$4$ & $1.0e$-$4$\\
			& wAIS & $9.2e$-$4$ & $4.6e$-$4$ & $2.2e$-$4$ & $1.6e$-$4$ & $9.0e$-$5$\\
			& AISCV & \bm{$3.2e$}-\bm{$4$} & \bm{$3.2e$}-\bm{$5$} & \bm{$1.1e$}-\bm{$5$} & \bm{$6.0e$}-\bm{$6$} & \bm{$2.5e$}-\bm{$6$}\\ 
			\hline
			\multirow{3}{3.1em}{\centering $g_2$ \\ $(d=4)$} &  AIS & $3.4e$-$4$ & $1.3e$-$4$ & $7.6e$-$5$ & $5.9e$-$5$ & $3.1e$-$5$\\
			& wAIS & $3.7e$-$4$ & $1.6e$-$4$ & $1.2e$-$4$ & $1.1e$-$4$ & $7.9e$-$5$\\
			& AISCV &  \bm{$3.1e$}-\bm{$5$} & \bm{$1.0e$}-\bm{$5$} & \bm{$4.9e$}-\bm{$6$} &\bm{$2.6e$}-\bm{$6$} & \bm{$1.5e$}-\bm{$6$}\\ 
			\hline
			\multirow{3}{3.1em}{\centering $g_3$ \\ $(d=8)$} &  AIS & $1.6e$-$3$ & $7.8e$-$4$ & $4.0e$-$4$ & $3.3e$-$4$ & $1.9e$-$4$ \\
			&  wAIS & $1.5e$-$3$ & $7.3e$-$4$ & $3.6e$-$4$ & $2.7e$-$4$ & $1.5e$-$4$ \\
			& AISCV & \bm{$1.7e$}-\bm{$4$} & \bm{$2.1e$}-\bm{$5$} & \bm{$7.8e$}-\bm{$6$} & \bm{$4.3e$}-\bm{$6$} & \bm{$1.8e$}-\bm{$6$}\\ 
			\hline
		\end{tabular}
		\caption{Mean Square Errors for $g_1,g_2,g_3$ with AIS, wAIS \citep{portier2018asymptotic} and AISCV in dimensions $d \in \{4;8\}$ obtained over $100$ replications.}
		\label{tab:synthetic_1}
	\end{table}
	
	\subsection{Synthetic examples: gaussian mixtures}
	
\textbf{General target $f$ and Stein method.} In this setting we only assume acces to the evaluations of the target density $f$. % and rely on Stein's method for the control variates. 
We consider the classical example introduced in \cite{cappe2008adaptive} where $f$ is a mixture of two gaussian distributions. The control variates are built using Stein's method with polynomial maps of degree $Q \in \{2;3\}$ leading to a number of control variates $m \in \{14;34\}$ in dimension $d=4$ and $m \in \{44;164\}$ in dimension $d=8$ respectively. 
	
	\textit{Isotropic case.} Let $f_{\Sigma}(x) = 0.5 \Phi_{\Sigma}(x-\mu) + 0.5 \Phi_{\Sigma}(x+\mu)$ where $\mu = (1,\ldots,1)^\top/2\sqrt{d}, \Sigma = I_d/d$ and $\Phi_{\Sigma}$ is the multivariate normal density function with zero mean and covariance matrix $\Sigma$. Note that the Euclidean distance between the two mixture centers is independent of the dimension as it equals $1$. The initial density $q_0$ is the multivariate student distribution with mean $(1,-1,0,\ldots,0)/ \sqrt{d}$ and variance $(5/d)I_d$. The initial mean value differs from the null vector to prevent the naive algorithm using the initial density from having good results (due to the symmetry). 
	
	\textit{Anisotropic case.} In this case, the mixture is unbalanced and each gaussian is anisotropic. The target density is $f_{V}(x) = 0.75 \Phi_{V}(x-\mu) + 0.25 \Phi_{V}(x+\mu)$ where $\mu = (1,\ldots,1)^\top/2\sqrt{d}$ and $V = Diag(10,1,\ldots,1)/d$. The initial density $q_0$ is the same as for the isotropic case.
	
Figure \ref{fig:app_synthetic_2} presents the boxplots of the mean square error $\| \hat I(g) - I(g)\|_2^2$ and Table \ref{tab:synthetic_2} gathers the associated numerical values.
	
		\begin{figure}[h]
		\centering
		\begin{subfigure}[b]{0.245\textwidth}
			\centering
			\includegraphics[width=\textwidth]{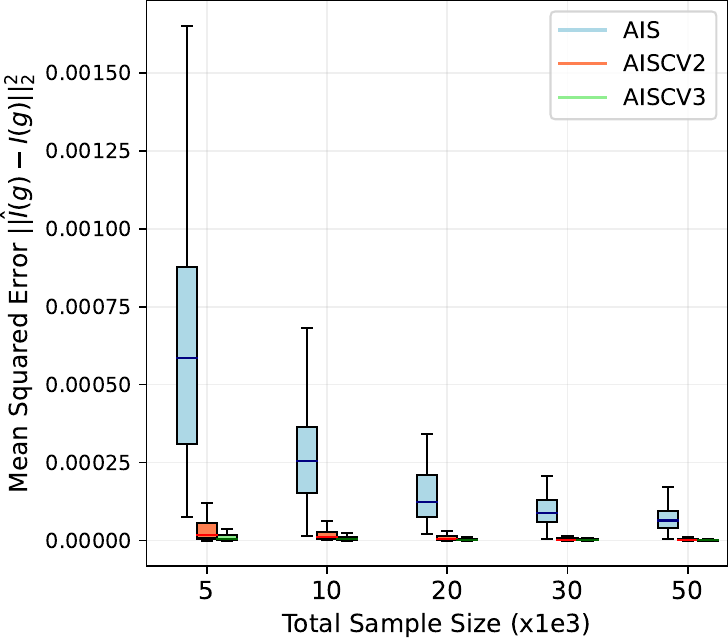}
			\caption{$f_{\Sigma} (d=4)$}
			\label{fig:iso_d4}
		\end{subfigure}
		\hfill
		\begin{subfigure}[b]{0.245\textwidth}
			\centering
			\includegraphics[width=\textwidth]{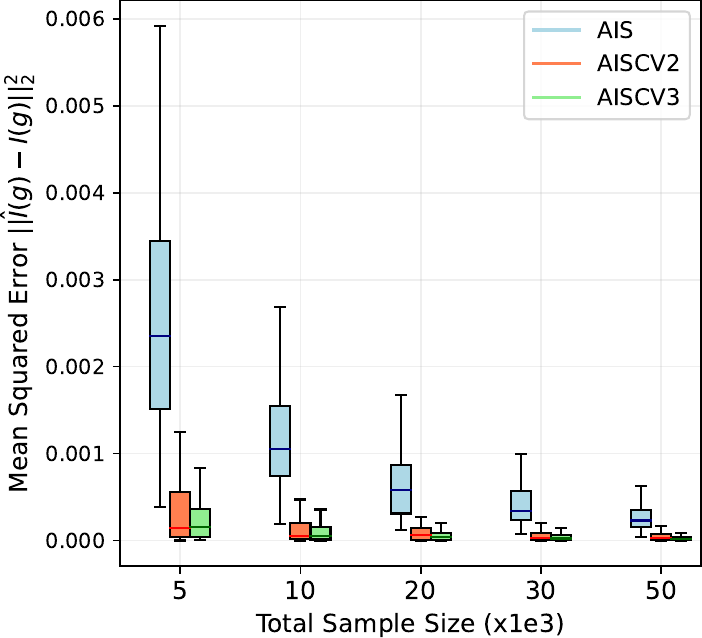}
			\caption{$f_{\Sigma} (d=8)$}
			\label{fig:iso_d8}
		\end{subfigure}
		\hfill
		\begin{subfigure}[b]{0.245\textwidth}
			\centering
			\includegraphics[width=\textwidth]{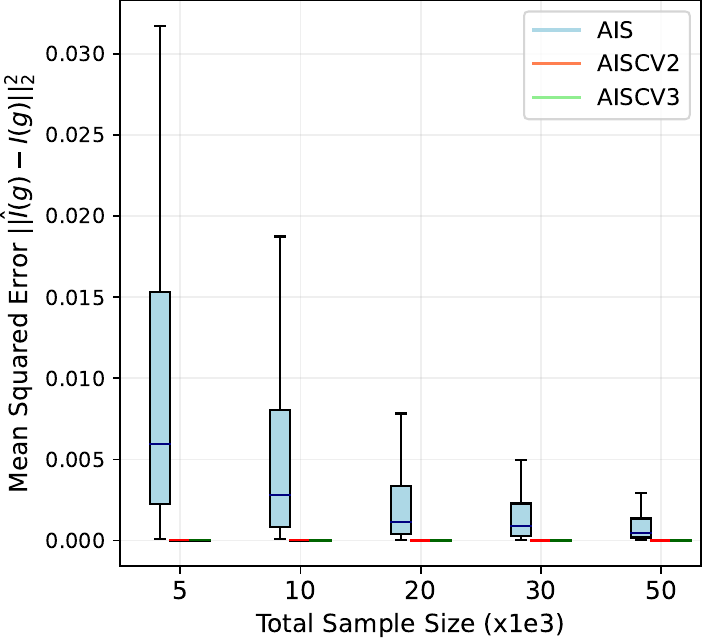}
			\caption{$f_V (d=4)$}
			\label{fig:ani_d4}
		\end{subfigure}
		\hfill
		\begin{subfigure}[b]{0.245\textwidth}
			\centering
			\includegraphics[width=\textwidth]{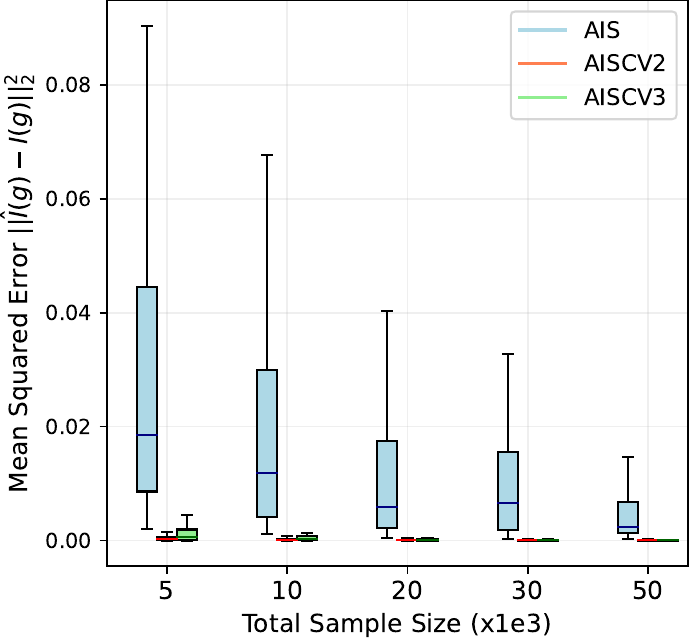}
			\caption{$f_{V} (d=8)$}
			\label{fig:ani_d8}
		\end{subfigure}
		\caption{Boxplots for $\| \hat I(g) - I(g)\|_2^2$ for $g(x)=x$ with target isotropic $f_{\Sigma}$ and anisotropic $f_V$ in dimensions $d \in \{4;8\}$ obtained over $100$ replications.}
		\label{fig:app_synthetic_2}        
	\end{figure}
	
	\begin{table}[h!]
		\centering
		\begin{tabular}{ccccccc} 
			\hline
			\multicolumn{2}{c}{Sample Size $n$} & \multirow{2}{2.7em}{$5,000$} & \multirow{2}{2.7em}{$10,000$} & \multirow{2}{2.7em}{$20,000$} & \multirow{2}{2.7em}{$30,000$} & \multirow{2}{2.7em}{$50,000$} \\ 
			Target & Method & & & & & \\
			\hline
			\multirow{4}{3.1em}{\centering $f_\Sigma$ \\ $(d=4)$} &  AIS & $6.9e$-$4$ & $2.9e$-$4$ & $1.5e$-$4$ & $1.1e$-$4$ & $7.2e$-$5$ \\
			&  wAIS & $6.8e$-$4$ & $2.9e$-$4$ & $1.5e$-$4$ & $1.1e$-$4$ & $7.3e$-$5$ \\
			& AISCV-2 & $4.1e$-$5$ & $2.2e$-$5$ & $9.1e$-$6$ & $5.6e$-$6$ & $3.7e$-$6$ \\ 
			& AISCV-3 & \bm{$1.5e$}-\bm{$5$} & \bm{$8.4e$}-\bm{$6$} & \bm{$3.7e$}-\bm{$6$} & \bm{$2.3e$}-\bm{$6$} & \bm{$1.3e$}-\bm{$6$} \\ 
			\hline
			\multirow{4}{3.1em}{\centering $f_\Sigma$ \\ $(d=8)$} &  AIS & $2.7e$-$3$ & $1.2e$-$3$ & $6.6e$-$4$ & $4.1e$-$4$ & $2.7e$-$4$\\
			&  wAIS & $2.7e$-$3$ & $1.2e$-$3$ & $6.9e$-$4$ & $4.3e$-$4$ & $2.8e$-$4$\\
			& AISCV-2 & $3.7e$-$4$ & $1.7e$-$4$ & $1.0e$-$4$ & $6.8e$-$5$ & $4.7e$-$5$\\ 
			& AISCV-3 & \bm{$2.8e$}-\bm{$4$} & \bm{$1.2e$}-\bm{$4$} & \bm{$6.3e$}-\bm{$5$} & \bm{$4.2e$}-\bm{$5$} & \bm{$2.6e$}-\bm{$5$} \\ 
			\hline
			\multirow{4}{3.1em}{\centering $f_V$ \\ $(d=4)$} &  AIS & $1.1e$-$2$ & $5.5e$-$3$ & $2.2e$-$3$ & $1.6e$-$3$ & $9.5e$-$4$\\
			&  wAIS & $1.1e$-$2$ & $5.3e$-$3$ & $2.0e$-$3$ & $1.3e$-$3$ & $8.0e$-$4$\\
			& AISCV-2 &  $1.3e$-$5$ & $7.2e$-$6$ & $2.9e$-$6$ & $1.9e$-$6$ & $1.2e$-$6$\\ 
			& AISCV-3 & \bm{$1.1e$}-\bm{$5$} & \bm{$6.6e$}-\bm{$6$} & \bm{$2.2e$}-\bm{$6$} & \bm{$1.5e$}-\bm{$6$} & \bm{$9.6e$}-\bm{$7$} \\ 
			\hline
			\multirow{4}{3.1em}{\centering $f_V$ \\ $(d=8)$} &  AIS & $4.5e$-$2$ & $3.2e$-$2$ & $2.2e$-$2$ & $1.5e$-$2$ & $6.8e$-$3$ \\
			&  wAIS & $2.6e$-$2$ & $1.3e$-$2$ & $7.8e$-$3$ & $5.9e$-$3$ & $3.8e$-$3$ \\
			& AISCV-2 & \bm{$4.6e$}-\bm{$4$} & \bm{$2.8e$}-\bm{$4$} & \bm{$1.3e$}-\bm{$4$} & \bm{$9.7e$}-\bm{$5$} & $6.0e$-$5$ \\ 
			& AISCV-3 & $1.4e$-$3$ & $4.8e$-$4$ & $1.5e$-$4$ & $1.1e$-$4$ & \bm{$5.7e$}-\bm{$5$} \\ 
			\hline
		\end{tabular}
		\caption{Mean Square Errors $\| \hat I(g) - I(g)\|_2^2$ for $g(x)=x$ with target isotropic $f_{\Sigma}$ and anisotropic $f_V$ in dimensions $d \in \{4;8\}$ obtained over $100$ replications.}
		\label{tab:synthetic_2}
	\end{table}
	
	\subsection{Real-world data: Bayesian linear regression}

We place ourselves in the framework of Bayesian linear regression with observations  $X \in \rset^{N \times d}$ and labels $y \in \rset^N$. The posterior distribution $p(\theta|\mathcal{D})$ depends on a gaussian prior $\pi \sim \mathcal{N}(\mu_a,\Sigma_a)$ and a likelihood function $\ell(\mathcal{D}|\theta) \propto (\sigma^2)^{-N/2} \exp(-(y-X \theta)^\top (y-X \theta)/(2 \sigma^2))$ where the noise level is fixed and taken sufficiently large $\sigma = 50$ to account general priors. Observe that the posterior distribution is also gaussian $\mathcal{N}(\mu_b,\Sigma_b)$ with $\mu_b = \Sigma_b (\sigma^{-2} X^\top y + \Sigma_a^{-1} \mu_a)$ and $\Sigma_b = (\sigma^{-2} X^\top X + \Sigma_a^{-1})^{-1}$. The integrand is $g(\theta) = \|\theta\|_2^2$ and the control variates are built using Stein method described in Section \ref{subsec:cv_practice} with degree $Q \in \{1;2\}$, leading to the AISCV1 and AISCV2 estimators respectively. Observe that when $Q=2$, the integrand belongs to the linear span of the control variates so the integration should be exact in light of Proposition \ref{prop:exact}.

\begin{figure}[h!]
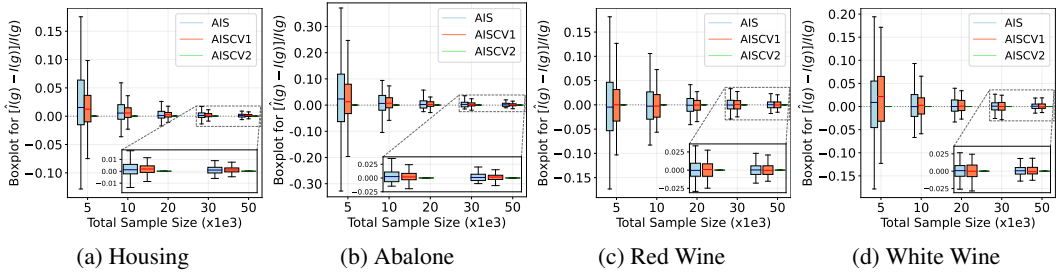

		\centering
		\begin{subfigure}[b]{0.245\textwidth}
			\centering
			\includegraphics[width=\textwidth]{housing_norm.pdf}
			\caption{Housing}
			\label{fig:housing}
		\end{subfigure}
		\hfill
		\begin{subfigure}[b]{0.245\textwidth}
			\centering
			\includegraphics[width=\textwidth]{abalone_norm.pdf}
			\caption{Abalone}
			\label{fig:abalone}
		\end{subfigure}
		\hfill
		\begin{subfigure}[b]{0.245\textwidth}
			\centering
			\includegraphics[width=\textwidth]{redwine_norm.pdf}
			\caption{Red Wine}
			\label{fig:mpg}
		\end{subfigure}
		\hfill
		\begin{subfigure}[b]{0.245\textwidth}
			\centering
			\includegraphics[width=\textwidth]{whitewine_norm.pdf}
			\caption{White Wine}
			\label{fig:trya}
		\end{subfigure}
		\caption{Boxplots of $(\hat I(g) - I(g))/I(g)$,  $g(\theta)=\|\theta\|_2^2$, obtained over $100$ replications.}
		\label{fig:app_BLR_real}        
	\end{figure}	
	
\textbf{Datasets and parameters.} Some classical datasets from UCI Machine Learning repository \cite{dua2019uci} are considered : \textit{housing} ($N=506;d=13;m \in \{12;104\}$); \textit{abalone} ($N=4,177;d=8;m \in \{7;44\}$); \textit{red wine} ($N=1,599;d=11;m \in \{10;77\}$) and \textit{white wine} ($N=4,898;d=11;m \in \{10;77\}$). The initial density is the multivariate student distribution with $\nu = 10$ degrees of freedom, zero mean and covariance matrix $\Sigma_b$.

\textbf{Results.} Figure \ref{fig:app_BLR_real} presents the boxplots of the error $(\hat I(g) - I(g))/I(g)$ and Table \ref{tab:BLR_real} gathers the associated numerical values. Observe the benefits of using control variates even with polynomials of degree $Q=1$. Observe that when $Q=2$, the error of the AISCV2 estimator is almost equal to zero which is in line with Proposition \ref{prop:exact}. Accordingly when looking at the MSE, the AISCV1 error is smaller than the AIS one by a factor ranging between $2$ and $10$ and the MSE of AISCV2 is of order $10^{-9}$.

	\begin{table}[h!]
		\centering
		\begin{tabular}{ccccccc} 
			\hline
			\multicolumn{2}{c}{Sample Size $n$} & \multirow{2}{2.7em}{$5,000$} & \multirow{2}{2.7em}{$10,000$} & \multirow{2}{2.7em}{$20,000$} & \multirow{2}{2.7em}{$30,000$} & \multirow{2}{2.7em}{$50,000$} \\ 
			Dataset & Method & & & & & \\
			\hline
			\multirow{3}{3.7em}{\centering Housing} &  AIS & $2.2e$-$2$ & $4.4e$-$3$ & $3.1e$-$4$ & $2.7e$-$4$ & $2.5e$-$4$ \\
			& AISCV1 & $2.9e$-$3$ & $7.0e$-$4$ & $1.7e$-$4$ & $1.6e$-$4$ & $5.2e$-$5$ \\
			& AISCV2 & $5.6e$-$9$ & $5.6e$-$9$ & $5.6e$-$9$ & $5.6e$-$9$ & $5.6e$-$9$ \\
			\hline
			\multirow{3}{3.7em}{\centering Abalone} &  AIS & $6.2e$-$2$ & $2.6e$-$2$ & $1.1e$-$2$ & $6.5e$-$3$ & $3.1e$-$3$\\
			& AISCV1 & $6.3e$-$3$ & $1.2e$-$3$ & $4.7e$-$4$ & $3.1e$-$4$ & $1.8e$-$4$ \\ 
			& AISCV2 & $5.1e$-$9$ & $6.1e$-$9$ & $6.1e$-$9$ & $6.1e$-$9$ & $6.1e$-$9$ \\
			\hline
			\multirow{3}{4.1em}{\centering Red  Wine} &  AIS & $3.0e$-$2$ & $1.3e$-$2$ & $7.0e$-$3$ & $4.7e$-$3$ & $2.8e$-$3$\\
			& AISCV1 & $3.7e$-$3$ & $1.5e$-$3$ & $8.7e$-$4$ & $6.4e$-$4$ & $4.2e$-$4$  \\ 
			& AISCV2 & $5.1e$-$10$ & $5.1e$-$10$ & $5.1e$-$10$ & $5.1e$-$10$ & $5.1e$-$10$ \\
			\hline
			\multirow{3}{4.8em}{\centering White Wine} &  AIS & $1.1e$-$2$ & $2.6e$-$3$ & $8.1e$-$4$ & $4.2e$-$4$ & $1.8e$-$4$ \\
			& AISCV1 & $7.1e$-$3$ & $1.5e$-$3$ & $4.0e$-$4$ & $2.1e$-$4$ & $9.2e$-$5$\\ 
			& AISCV2 & $2.4e$-$9$ & $2.4e$-$9$ & $2.4e$-$9$ & $2.4e$-$9$ & $2.4e$-$9$ \\
			\hline
		\end{tabular}
		\caption{Mean Square Errors for different datasets with $g(\theta)=\|\theta\|_2^2$ obtained over $100$ replications.}
		\label{tab:BLR_real}
	\end{table}
	
		\begin{table}[h!]
		\centering
		\begin{tabular}{ccccc} 
			\hline
			Sample Size $n$ & $5,000$ & $10,000$ & $20,000$ & $30,000$ \\
			\hline
			Housing &  $5.5e$-$5$ & $2.7e$-$5$ & $1.9e$-$5$ & $1.2e$-$5$  \\
			\hline
			Abalone &  $1.8e$-$4$ & $8.6e$-$5$ & $6.7e$-$5$ & $5.6e$-$5$  \\
			\hline
			Red Wine &  $2.7e$-$4$ & $1.8e$-$4$ & $9.5e$-$5$ & $5.2e$-$5$  \\
			\hline
			White Wine &  $3.8e$-$4$ & $1.6e$-$4$ & $8.5e$-$5$ & $7.3e$-$5$  \\
			\hline
		\end{tabular}
		\caption{MSE with $g(\theta)=\|\theta\|_2^2$ obtained over $30$ chains of NUTS sampler.}
		\label{tab:results_HMC}
	\end{table}
	
\textbf{Monte Carlo Markov Chain}. We run a state-of-the-art MCMC method called NUTS \cite{hoffman2014no}, which is a self-tuning variant of Hamiltonian Monte Carlo. It may be hard to compare precisely this method against the AIS based methods since they are different in nature. Indeed, the goal of MCMC methods is to sample from a target distribution whereas AISCV methods are meant for variance reduction. In both cases there are hyperparameters to tune. For AIS-based methods, there is the choice of the policy $(q_i)_{i \geq 0}$, the choice of the control variates and the number of particles $n_t$ to draw at each step. For the NUTS sampler, there is among others, the number of samples used for the tuning phase and the initialization of the Markov kernel. A reasonable comparison is obtained based on the overall number of sampled particles. Table \ref{tab:results_HMC} above presents the mean squared errors obtained over $30$ chains of NUTS sampler with default configuration of the parameters from the Python library \textit{pymc3} \cite{patil2010pymc}.

\newpage
%%%%%%%%%%%%%%%%%%%%%%%%%%%%%%%%%%%%%%%%%%%%%%%%%%%%%%%%%%%%
%%%%%%%%%%%%%%%%%%%%%%%%%%%%%%%%%%%%%%%%%%%%%%%%%%%%%%%%%%%%

%\appendix

%\section{Appendix}

%Optionally include extra information (complete proofs, additional experiments and plots) in the appendix.
%This section will often be part of the supplemental material.

\end{document}